\documentclass{article}

\PassOptionsToPackage{numbers,compress,sort}{natbib}

\usepackage[final]{neurips_2022}

\usepackage[utf8]{inputenc} 
\usepackage[T1]{fontenc}    
\usepackage{url}            
\usepackage{booktabs}       
\usepackage{amsfonts}       
\usepackage{nicefrac}       
\usepackage{microtype}      
\usepackage{xcolor}         

\usepackage[colorlinks=true,citecolor=blue,urlcolor=blue,linkcolor=blue]{hyperref}       
\usepackage{physics}
\usepackage{multirow}
\usepackage{graphicx}
\usepackage{subfig}

\usepackage[textsize=scriptsize]{todonotes}
\usepackage{bm}

\usepackage[labelfont=bf,textfont=md,font=small]{caption}
\usepackage{amsthm,amsmath,mathtools,amssymb}

\renewcommand{\paragraph}[1]{\textbf{#1}\hskip 6pt}

\theoremstyle{plain}
\newtheorem{thm}{Theorem}[]

\theoremstyle{definition}

\theoremstyle{remark}

\usepackage[nameinlink,capitalize]{cleveref}
\Crefname{con}{Conjecture}{Conjectures}

\title{Combining Explicit and Implicit Regularization for Efficient Learning in Deep Networks}

\author{%
  Dan Zhao \\
  New York University \\
  \texttt{dz1158@nyu.edu} \\
}

\begin{document}

\maketitle

\begin{abstract}
Works on implicit regularization have studied gradient trajectories during the optimization process to explain why deep networks favor certain kinds of solutions over others. In deep linear networks, it has been shown that gradient descent implicitly regularizes toward low-rank solutions on matrix completion/factorization tasks. Adding depth not only improves performance on these tasks but also acts as an accelerative pre-conditioning that further enhances this bias towards low-rankedness. Inspired by this, we propose an explicit penalty to mirror this implicit bias which only takes effect with certain adaptive gradient optimizers (e.g. Adam). This combination can enable a degenerate single-layer network to achieve low-rank approximations with generalization error comparable to deep linear networks, making depth no longer necessary for learning. The single-layer network also performs competitively or out-performs various approaches for matrix completion over a range of parameter and data regimes despite its simplicity. Together with an optimizer’s inductive bias, our findings suggest that explicit regularization can play a role in designing different, desirable forms of regularization and that a more nuanced understanding of this interplay may be necessary. 
\end{abstract}

\section{Introduction}
Much work has poured into understanding why and how highly over-parameterized, deep neural networks with more parameters than training examples generalize so effectively despite long-held notions to the contrary \cite{Belkin:2018wm,Bartlett:2020ck}. This generalization puzzle has only deepened as deep learning models often generalize well simply by optimizing its training error on an under-determined problem.

To explain this, previous works have focused on how gradient-based optimization induces an implicit bias on optimization trajectories, particularly in deep (i.e., over-parameterized) settings, tending towards solutions with certain properties \cite{Gunasekar:2017vv,Arora:2019ug} like those that generalize well. In contrast, while explicit regularization has seen wide-spread usage in various settings (e.g. weight decay, dropout \cite{Srivastava:2014ww}), its role in explaining generalization has been less certain given its inability to prevent over-fitting on random labels \cite{Zhang:2016ve} or its absence in deep models that generalize well on their own.

Some works have focused on a simple test-bed to formalize and isolate the mechanisms through which implicit regularization operates---namely, \textit{matrix completion}. Given some observed subset of an unknown, low-rank matrix $W^\star$, the task is to recover the unseen entries. A key observation \cite{Gunasekar:2017vv} has been how gradient descent on a shallow linear neural network, with sufficiently small learning rate and near-zero initialization, pushes towards low-rank solutions on its own. This has led to the conjecture \cite{Gunasekar:2017vv} that gradient descent induces implicit regularization that minimizes the nuclear norm.

This conjecture has been put into doubt by work \cite{Arora:2019ug}  showing that gradient descent not only promotes low-rank solutions in the shallow case on matrix completion tasks, but its implicit regularization is further strengthened with increased depth---deep linear neural networks \cite{Arora:2019ug}) are able to produce solutions with lower rank and more accurate completion than those from minimizing the nuclear norm. Others \cite{Arora:2018vn} have also shown how increased depth or over-parameterization can provide an accelerative pre-conditioning to regression problems for faster convergence in deep linear networks.

\paragraph{Contributions}
Despite these findings, some questions still remain: 
\begin{itemize}
  \item Are there explicit norm-based penalties that can mirror the effect of implicit regularization so as to better harness this phenomenon for improved or more efficient learning? Previous work \cite{Arora:2019ug} has conjectured that implicit regularization cannot be characterized by explicit norm-based penalties, but whether these penalties can produce similar effects is unclear. 
  
  \item Do implicit and explicit forms of regularization interact in any meaningful way? Can we modify the implicit biases of optimizers with explicit regularizers so as to promote better kinds of performance? Some work \cite{barrett2021gradImpReg} has begun to draw inspiration from implicit regularization to create explicit regularizers, but their interactions are less clear.
  
  \item Previous works \cite{Arora:2018vn,Arora:2019ug} have shown that depth can act as a powerful pre-conditioning to accelerate convergence or enhance implicit tendencies towards certain simpler or well-generalizing solutions. Can this effect be produced without depth? 
\end{itemize}

To try and shed more light on these questions, we propose an explicit penalty that takes the ratio between the nuclear norm of a matrix and its Frobenius norm ($\norm{W}_{\star}/\norm{W}_{F}$) and study its effects on the task of matrix completion. This penalty can be interpreted as an adaptation of the Hoyer measure \cite{hurley2009comparing} to the spectral domain or as a particular normalization of the nuclear-norm penalty that is commonly used to proxy for rank in convex relaxations of the problem.

Studying implicit regularization can be difficult as it is not always possible to account for all other sources of implicit regularization. For a more precise study of the implicit biases of optimizers and their interactions with our penalty, we use matrix completion and deep linear networks as tractable, yet expressive, and well-understood test-beds that admit a crisp formulation of the mechanism through which implicit regularization operates \cite{gunasekar2018implicit, Arora:2019ug}. In particular, we show the following:

\begin{enumerate}
    \item A depth 1 linear neural network (i.e., a degenerate network without any depth) trained with this penalty can produce the same rank reduction and deliver comparable, if not better, generalization performance than a deep linear network---all the while converging faster. In short, depth is no longer necessary for learning.
  
    \item The above result only occurs under Adam and, to some extent, its close variants. This suggests that different optimizers, each with their own inductive biases, can interact differently with explicit regularizers to modify dynamics and promote certain solutions over others.

    \item With the penalty, we achieve comparable or better generalization and rank-reduction performance against various other techniques (\cref{fig:combo}) even in low data regimes (i.e., fewer observed entries during training) where other approaches may have no recovery guarantees.

    \item Furthermore, the penalty under Adam enables linear neural networks of all depth levels to produce similar well-generalizing low-rank solutions largely independent of depth, exhibiting a degree of \textit{depth invariance}.
\end{enumerate}

In this specific case, it appears that the learning dynamics which occur through the inductive bias of depth can be compressed or replaced with the right combination of optimization algorithm and explicit penalty. These properties may make deep linear networks more efficient for a variety of related matrix completion, estimation, or factorization tasks, ranging from applications in efficient reinforcement learning \cite{2020MEQLearn} to adversarial robustness \cite{2019MEnet}, NLP \cite{2019MatFactNLPEmbed}, and others.

Previous conjectures and works \cite{Arora:2018vn, Arora:2019ug, Zhang:2016ve} have largely dismissed the necessity of explicit regularization in understanding generalization in deep learning, leaving its overall role unclear. Similarly, despite its relative popularity, Adam and other adaptive optimizers have received their share of doubt \cite{wilson2017adaptgradient,wadia2020destroy} regarding their effectiveness in producing desirable solutions. Our results suggest a subtle but perhaps important interplay between the choice of the optimizer, its inductive bias, and the explicit penalty in producing different optimization dynamics and, hence, different kinds of solutions. If so, then in these interactions, explicit regularization may yet have a role to play. 

\paragraph{Paper Overview} In \cref{sec:related_work}, we review previous work. \cref{sec:findings} details our experiments and findings. In \cref{sec:empirical}, we extend our experiments to a common real-world benchmark and compare our method with other methods. We conclude in \cref{sec:discuss} with a discussion on future work.

\section{Related Work}
\label{sec:related_work}

Implicit regularization has been studied extensively with recent work focusing on optimization trajectories in deep, over-parameterized settings \cite{Arora:2019ug, Arora:2018vn, bah2019learning}. 
One avenue in particular has focused on linear neural networks (LNNs) \cite{saxe2013exact,goodfellow2016deep,hardt2016identity,kawaguchi2016deep} given their relative tractability and the similarity of their learning dynamics to those of non-linear networks. In settings with multiple optima, \cite{gunasekar2018implicit} has shown how LNNs trained on separable data can converge with an implicit bias to the max margin solution. Others \cite{Arora:2019ug} have demonstrated how gradient flow converges towards low-rank solutions where depth acts as an accelerative pre-conditioning \cite{Arora:2018vn}. In \cite{precondgen}, for natural and vanilla gradient descent (GD), different kinds of pre-conditioning are shown to impact bias-variance and risk trade-offs in over-parameterized linear regression, but this is less clear for adaptive gradient optimizers. 

Building upon notions of acceleration and pre-conditioning dating back to Nesterov \cite{Nesterov:2013gs} and Newton, Adam's \cite{Kingma:2015us} effectiveness---and its close variants \cite{nadam, radam, adamw}---in optimizing deep networks faster makes clear the importance of adaptive pre-conditioning. Though some \cite{wilson2017adaptgradient, wadia2020destroy} have doubted their effectiveness due to potential issues that can harm generalization, others \cite{xu2020secondorder,levy2019geogradient,zhang2019abatchalgo, heusel2017GANlocalNash} have demonstrated advantages from their adaptive preconditioning; despite speeding up optimization, however, the effect of pre-conditioning on generalization has been less clear as some \cite{keskar2016gengap, dinh2017sharp, Wu2018dyamic} have argued that the ``sharpness'' of minima achieved can vary depending on the choice of optimizer. More recent works have characterized pre-conditioning and mechanisms of implicit regularization around ``edges of stability'' for optimizers where the training regime occurs within a certain sharpness threshold, defined as the maximum eigenvalue of the loss Hessian \cite{arora2018EOS, cohen2021EOS, cohen2022adaptiveEOS}.

Matrix completion and factorization have themselves long been an important focus for areas like signal recovery \cite{Candes:kr} and recommendation systems \cite{bennett2007netflix} from theoretical bounds for recovery and convergence \cite{candes2010power, Recht:2011up, candes2009exact,ma2019implicit} to practical algorithms and implementations \cite{softimpute2010,Jain:2013eb}. We refer to \cite{Chi:2019hy} for a comprehensive survey. These tasks and related ones have also served as test-beds for understanding implicit regularization; \cite{Arora:2019ug, Gunasekar:2017vv} use matrix factorization and sensing to study gradient flow's implicit bias towards low-rank solutions, conjecturing that algorithmic regularization may not correspond to minimization of any norm-based penalty. \cite{wu2021mirror} studies the implicit bias of mirror descent on matrix sensing, showing that the solution interpolates between the nuclear and Frobenius norms depending on the underlying matrix. In a related thread, \cite{arora2022impRegMirror} has shown that gradient flow with any commuting parameterization is equivalent to continuous mirror descent with a specific Legendre function, generalizing previous results  \cite{gunasekar2020KernelOver, Gunasekar:2017vv, woodworth2021ImpBiasInit} that have characterized the implicit bias of GD.

Interestingly, \cite{barrett2021gradImpReg} illustrates how the discrete steps of GD can regularize optimization trajectories away from large gradients towards flatter minima, developing an explicit regularizer to embody and reinforce this implicit effect directly. To our knowledge, our paper is the first to propose a ratio penalty in matrix completion to study the interplay between explicit and implicit regularization. 

\section{Experiments and Findings}
\label{sec:findings}
\subsection{Setup}
Formally, we have a ground-truth matrix $W^{\star} \in \mathbb{R}^{m \times n}$ whose observed entries are indexed by the set $\Omega$. We define the projection $\mathcal{P}_{\Omega}(W^{\star})$ to be a $m \times n$ matrix such that the entries with indices in $\Omega$ remain while the rest are masked with zeros. We are interested in the following optimization:
\begin{align}
  \min_{W} \mathcal{L}(W) \coloneqq \min_{W} \norm{ \mathcal{P}_{\Omega}(W^{\star}) - \mathcal{P}_{\Omega}(W) }_{F}^2 + \lambda R(W)
  \label{eqn:setup}
\end{align}
where $\norm{\,\cdot\,}_{F}$ is the Frobenius norm, $R(\cdot)$ is an explicit penalty, and $\lambda \geq 0$ is the tuning parameter. While we consider various penalties, our main focus is demonstrating the effects of our proposed penalty $R(W) = \vert \vert W \vert \vert_*/\vert \vert W \vert \vert_F$. 
Following earlier works \cite{Arora:2015vt}, we define a \textit{deep linear neural network} (DLNN) through the following over-parameterization, or deep factorization, of $W$:
\begin{align}
  W = W_{N}W_{N-1}\ldots W_1
  \label{eqn:factorize}
\end{align}
under the loss function in \eqref{eqn:setup} where $W_i \in \mathbb{R}^{d_i \times d_{i-1}}, i \in \{1, \ldots, N\}$ denotes the weight matrix corresponding to depth $i$ or the $i$-th layer. Here, $N$ denotes the depth of the network/factorization where $N = 2$ corresponds to \textit{matrix factorization} or a shallow network, $N \geq 3$ corresponds to \textit{deep matrix factorization} or a deep network, and $N = 1$ is the degenerate case (no depth/factorization). We refer to the matrix $W$, the product of the $N$ weight matrices in \cref{eqn:factorize}, as the \textit{end-product matrix} as per \cite{Arora:2019ug}. As such, the end-product matrix $W$ is the solution produced in estimating $W^{\star}$ or, conveniently, the DLNN itself.

In our analyses, we focus on rank 5 matrices as the ground truth $W^*$ and parameterize our DLNN $W$ with $d_0 = \ldots = d_N = m = n = 100$ (i.e., weight matrices $W_i \in \mathbb{R}^{100 \times 100}, \ \forall i$) for illustrative purposes, but our results extend to other ranks (e.g. see \cref{sec:rank10}). We follow previous work \cite{Arora:2019ug} and employ the effective rank \cite{roy2007effective} of a matrix to track and quantify the rank of $W$ in our experiments, defined as: $\operatorname{e-rank}(W) = \exp\left\{ H(p_1, \ldots, p_n) \right\}$
where $H$ is the Shannon entropy, $p_i = \sigma_i/\norm{\sigma}_{1}$, $\{\sigma_i\}$ are the unsigned singular values of $W$, and $\norm{\cdot}_1$ is the $\ell_1$ norm. The numerical instability of the numeric rank measure is a known issue \cite{roy2007effective}, resulting in unreliable and unstable rank estimates. We leave a detailed discussion of experiment settings to \cref{sec:implementation_details}.



\subsection{Depth, without penalty}
We first establish a baseline by characterizing the inductive biases of un-regularized gradient descent and un-regularized Adam to better understand their dynamics in the presence of our penalty.

\begin{figure}[!ht]
    \centering
    \includegraphics[width=.9\linewidth]{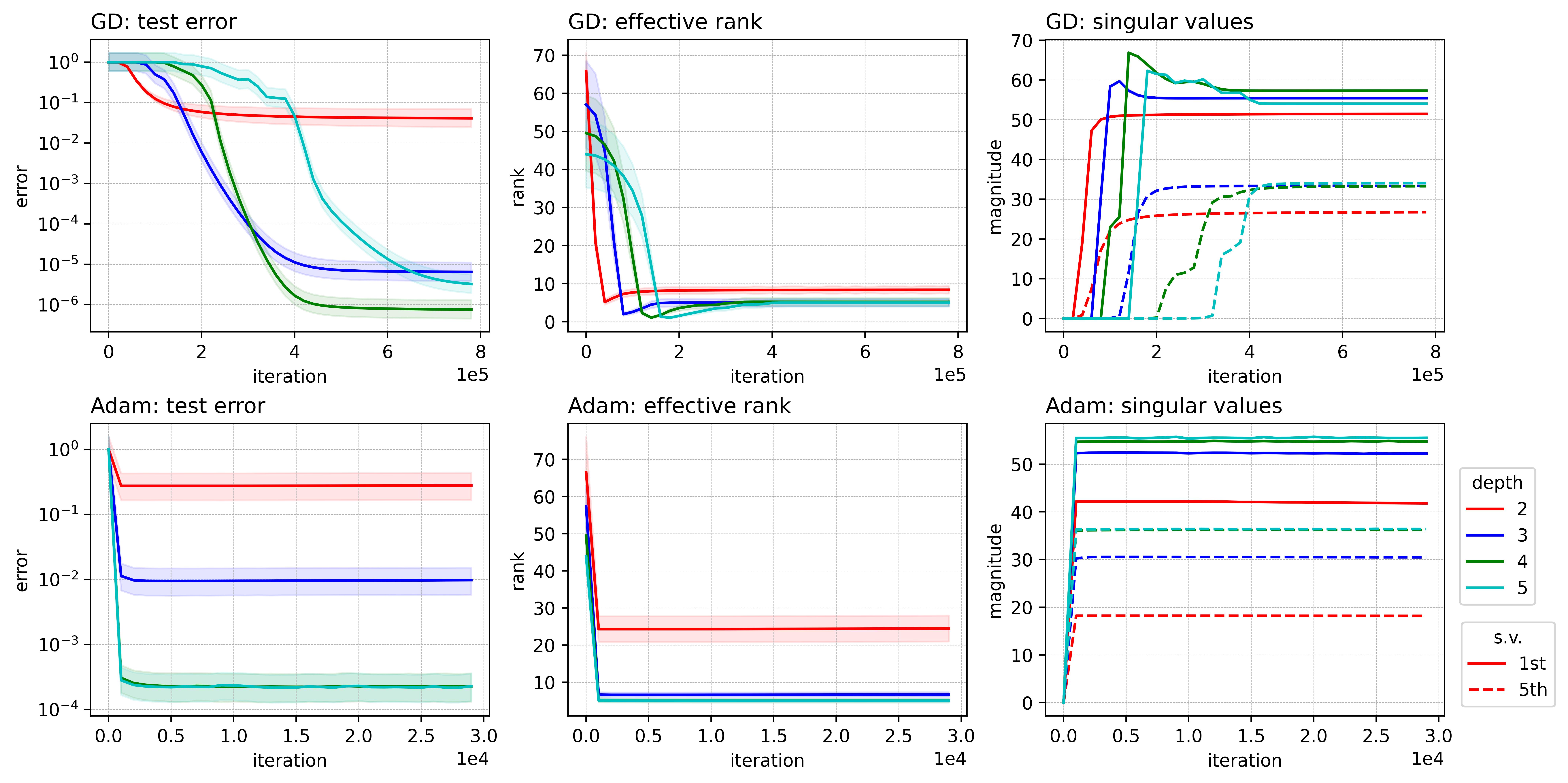}
    \caption{Dynamics of \textit{un-regularized gradient descent (GD) and Adam}. Plots show the performance of GD over networks of depths 2/3/4/5 for rank 5 matrices of size $100\times100$.  Colors correspond to different depth levels and shaded regions correspond to error bands. The left column depicts generalization error as a function of depth and training iterations. The middle column depicts the change in effective rank across depths and over training iterations. The right column shows the 1\textsuperscript{st} and 5\textsuperscript{th} largest singular values for each depth across training iterations. For singular values, a solid line indicates the 1\textsuperscript{st} largest singular value while a dotted line indicates the 5\textsuperscript{th} largest within each depth level (colored lines). We omit the remaining singular values to avoid clutter.}
    \label{fig:gd}
\end{figure}
\paragraph{Gradient Descent}
Previous work \cite{Arora:2019ug} has shown that depth enhances gradient descent's implicit regularization towards low rank, characterized by the following trajectories on the end-product matrix $W$ and its singular values $\{
\sigma_i\}$ (for details, see \cite{Arora:2018vn, Arora:2019ug}):
\vspace{-.7mm}
\begin{align}
\dot{\sigma}_i &= -N (\sigma_i(t)^2)^{\frac{N-1}{N}} \cdot \vb{u}_i^\top \nabla_W \mathcal{L}(W(t)) \vb{v}_i
\label{eqn:gd_sv} \\
\mathrm{vec}(\dot{W}) &= -P_{W} \mathrm{vec} \left(\nabla_W \mathcal{L}(W)\right)
\label{eqn:gd_w}
\end{align}
where $\dot{\sigma}_i$ is the time derivative of $\sigma_i(t)$, the $i$-th singular value of $W(t)$, $\{ \vb{u}_i, \vb{v}_i \}$ are the left and right singular vectors of $W(t)$ corresponding to $\sigma_i(t)$, $N$ is the network's depth, $\nabla_W \mathcal{L}(W(t))$ is the loss gradient with respect to the end-product matrix $W$ at time $t$,  $\text{vec}(\cdot)$ denotes (column-first order) vectorization, $\dot{W} = dW/dt$ is the time evolution of the end-product matrix or (equivalently) the DLNN itself, $P_{W} = \sum_{j=1}^N (W^\top W)^{\frac{N-j}{N}} \otimes (W W^\top)^{\frac{j-1}{N}}$, and $\otimes$ denotes the Kronecker product. We suppress the explicit dependency on $t$ for simplicity and note that full dynamics in \cref{eqn:gd_sv} require non-degeneracy (non-trivial depth, $N > 1$); otherwise, they reduce to just $\vb{u}_i^\top \nabla_W \mathcal{L}(W(t)) \vb{v}_i$. 

In \cref{eqn:gd_w}, $P_W$ can be seen as a pre-conditioning onto the gradient that, with sufficient depth ($N \geq 2)$, accelerates movements already taken in the optimization \cite{Arora:2018vn}. As depth/over-parameterization increases, this acceleration intensifies while larger singular values and their movements become more pronounced than their smaller counterparts, driving singular value separation and a decrease in rank of the recovered matrix (\cref{fig:gd} top row). The singular values evolve at uneven paces depending on the depth; increasing depth increases the gap in the time it takes between the 1\textsuperscript{st} and 5\textsuperscript{th} largest singular values to develop while also taking longer to  stabilize. These effects are even more pronounced when comparing the five largest singular values to the remaining ones. 
Only with sufficient depth ($N > 2$) do solutions produced by un-penalized gradient descent minimize rank so as to recover the rank of the underlying matrix and produce solutions with low test error. 

\paragraph{Adam}
Analyzing Adam can be difficult given its exponentially moving average of gradients; to simplify our analysis, we borrow some assumptions from \cite{Adam:sgn} to approximate Adam's dynamics via gradient flow by assuming that the discounted gradients can be well-approximated by their expectation. (see \cref{sec:math} for more details).

\begin{thm}
    Under the assumptions above and of \cite{Arora:2019ug}, the trajectory of the singular values $\sigma_i$ of the end-product matrix $W$ can be approximately characterized as:
    \begin{equation}
    \dot{\sigma}_i =  - \mathrm{vec}(\vb v_i \vb u_i^\top)^\top  P_{W, G} \mathrm{vec}(\nabla_{W} \mathcal{L}(W))
    \label{eqn:adam_svtraj}
    \end{equation}
    Similarly, the trajectory of the end-product matrix $W$ itself can be approximately characterized as:
    \begin{equation}
    \mathrm{vec}(\dot{W}) = -P_{W, G}\mathrm{vec} (\nabla_W \mathcal{L}(W))
    \label{eqn:adam_w}
    \end{equation}
where $P_{W,G} = \sum_{j=1}^N ( (W W^\top)^{\frac{j-1}{N}} \otimes (W^\top W)^{\frac{N-j}{N}}) G_j$ is p.s.d. and $G_j$ is a diagonal matrix for layers $j \in \{1,\hdots,N\}$. Specifically, $G_j = \mathrm{diag}(\mathrm{vec}(S_j))$, $[S_j]_{m,n} = [(\nabla_{W_j} \mathcal{L}(W)^2 + s^2_j)^{-1/2}]_{m,n}$, $\nabla_{W_j}\mathcal{L}(W) = \partial \mathcal{L}(W)/\partial W_j$ is layer $j$'s loss gradient, and $s^2_j = \mathrm{var}(\nabla_{W_j}\mathcal{L}(W))$. 
\label{thm:adam_W_noreg}
\end{thm}
\begin{proof}
See \cref{sec:math}.
\end{proof}
Via this approximation, the pre-conditioning induced by Adam can be characterized as a modification of gradient descent's $P_W$, which now normalizes each layer by the square-root of its squared layer-wise loss gradient $(\partial \mathcal{L}(W)/\partial W_j)^2$ and the gradient variance $s^2_j$, before summing across all depth levels. Unlike before, the variance of the loss gradient comes into play. Whereas before the pre-conditioning served as a purely accelerative effect that intensifies with depth, its normalization by the gradient variance of each layer $W_j$ can now either dampen or further accelerate the trajectory.

Empirically, we see that depth enhances the implicit bias towards low-rank solutions for both Adam and gradient descent albeit differently (\cref{fig:gd}, middle column); in deeper regimes ($N > 2$), Adam minimizes rank to exact/near-exact rank recovery more smoothly than gradient descent via faster ($10^4$ vs. $10^5$ iterations) and more uniform convergence (\cref{fig:gd}, bottom row). With Adam, singular value dynamics exhibit significantly more uniform evolution regardless of depth in contrast to gradient descent (\cref{fig:gd}, right), leading to different types of trajectories and solutions.
\subsection{Depth, with penalty}
\begin{figure}[!ht]
    \centering
    \includegraphics[width=.9\linewidth]{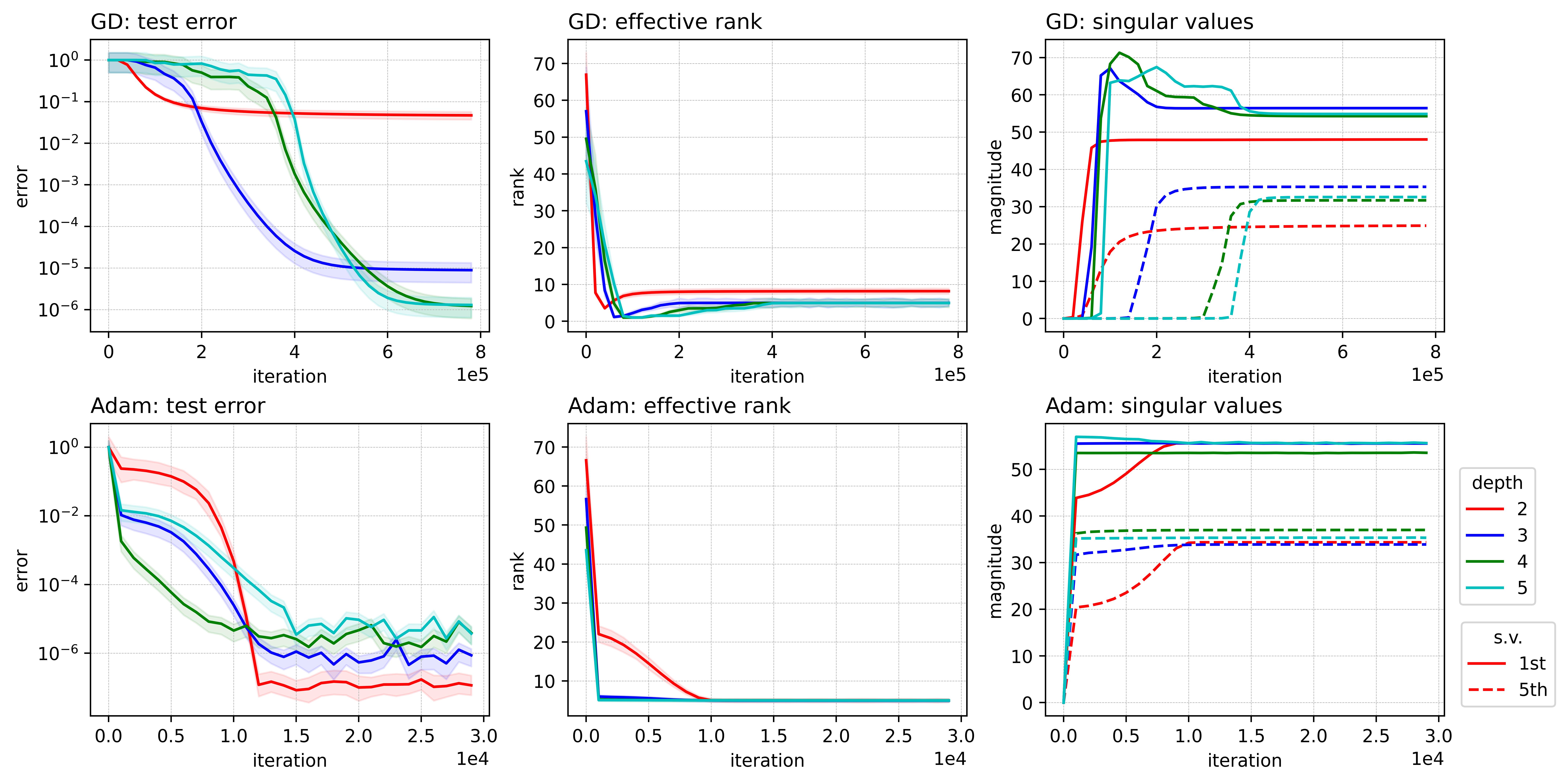}
    \caption{Dynamics of \textit{regularized gradient descent (GD) and regularized Adam} with our penalty. Plots above show the performance of GD with our proposed penalty over networks of depths 2/3/4/5 for rank 5 matrix completion. Setup is identical to \cref{fig:lm}. Here, $\lambda = 10^{-4}$  but results hold for a range of values $\lambda \in [10^{-4},10^{-1}]$.
    }
    \label{fig:gd+r} 
\end{figure}
\paragraph{Gradient Descent} We now characterize the dynamics of gradient flow with our penalty.
\begin{thm}
    Under the assumptions of \cite{Arora:2018vn}, the evolution of the singular values of the end-product matrix, under gradient descent with the penalty, can be approximated in the following fashion:
    \begin{equation}
     \dot{\sigma}_r = \dot{\sigma}_r^{\textrm{GF}} 
   - \frac{\lambda N}{\vert \vert W \vert \vert_F^2} \left(1 - \frac{\vert\vert W \vert\vert_*}{\vert\vert W \vert\vert_F}\right) \sigma_r^{\frac{3N - 2}{2}} 
   \label{eqn:gdreg_sv}
    \end{equation}
    where $\lambda \geq 0$ is the regularization parameter and $\dot{\sigma}_r^{\textrm{GF}}$ denotes the un-regularized singular value trajectory under gradient flow in \cref{eqn:gd_sv}.  
    Similarly, the evolution of $W$ can be approximated via: 
    \begin{equation}
    \mathrm{vec}(\dot{W}) = -P_{W}\left(\mathrm{vec} \left(\nabla_W \mathcal{L}(W)\right) + \lambda \frac{\mathrm{vec}(UV^\top - U \tilde{\Sigma} V^\top)}{\vert \vert W \vert \vert_F^2}\right) 
    \label{eqn:gdreg_w}
    \end{equation}
    where $U$, $V$ contain the left and right singular vectors of $W$, $P_W$ is the pre-conditioning in \cref{eqn:gd_w}, and $\tilde{\Sigma} = \frac{\vert \vert W \vert \vert_*}{\vert \vert W \vert \vert_F} \Sigma$ where $\Sigma$ contains the singular values of $W$ along its diagonal. \label{thm:gd_sv_reg}
\end{thm}
\vspace{-3mm}
\begin{proof}
    See \cref{sec:math} for details.
\end{proof}
\vspace{-3.5mm} The penalty can be seen as an additive component to the original trajectory in \cref{eqn:gd_sv} that also intensifies with increased depth, making movements in larger singular values more pronounced than those of smaller ones. As a result, it can enable more pronounced singular value separation than before ($\frac{2N-1}{N}$ vs. $\frac{3N-2}{2}$), depending on $\lambda$. Increasing depth continues to push down rank but, unlike before (\cref{eqn:gd_sv}, \cref{eqn:gd_w}), the penalty now allows singular value trajectories to depend on their own magnitudes even without depth (\cref{eqn:gdreg_sv} with $N=1$), providing an additional degree of freedom. The penalty also allows each singular value to depend on its relative weight within the distribution of singular values through $(1 - ||W||_*/||W||_F)$ rather than just its own absolute magnitude. 

In \cref{eqn:gdreg_w}, we also see that the depth-dependent accelerative pre-conditioning $P_W$ now acts on a new term: while the first term can be interpreted as the typical influence of reducing the loss via training and gradient optimization on the network's trajectory, the new term can be interpreted as a spectral-based component that can be used by $P_W$ to further enhance the spectral trajectory of $W$ at higher depths, like in \cref{eqn:gdreg_sv}. Looking at the diagonals, the new term can be seen as a spectrally re-scaled version of $W$ that influences its trajectory in a way that accounts for each singular value's weight relative to its entire spectrum: $\frac{\vb u_i^\top \vb v_i}{\vert \vert W \vert \vert_F ^2} (1 - \frac{\vert \vert W \vert \vert_*}{\vert \vert W \vert \vert_F}  \sigma_{i} )$. 

Empirically, comparing the un-regularized case (\cref{fig:gd} top row) to the regularized case (\cref{fig:gd+r} top row), we see that the penalty helps increase the speed at which rank is reduced, inducing faster rates of rank reduction earlier in training. Unlike in un-regularized gradient descent where deeper networks take longer to exhibit rank reduction, the penalty enables near simultaneous reductions in rank across all depths (\cref{fig:gd+r} top row, middle), making it less dependent on depth. 

\paragraph{Adam} With Adam, the penalty's effect differs considerably in terms of the solutions produced. 
\begin{thm}
    Under the same assumptions, with the proposed penalty, the evolution of the end-product matrix and its singular values can be approximated via the following: 
    \begin{equation}
    \vspace{-2mm}
    \begin{aligned}
    \dot{\sigma} &= - \mathrm{vec}(\vb v_i \vb u_i^\top)^\top P_{W,G} \left(\mathrm{vec} \left(\nabla_W \mathcal{L}(W)\right) + \lambda \frac{\mathrm{vec}(UV^\top - U \tilde{\Sigma} V^\top)}{\vert \vert W \vert \vert_F^2}\right)\\    
    \mathrm{vec}(\dot{W}) &= -P_{W,G}\left(\mathrm{vec} \left(\nabla_W \mathcal{L}(W)\right) + \lambda \frac{\mathrm{vec}(UV^\top - U \tilde{\Sigma} V^\top)}{\vert \vert W \vert \vert_F^2}\right) 
    \end{aligned} \label{eqn:adamreg}
    \end{equation}
\end{thm}
\begin{proof}
\vspace{-1mm}
    Follows from \cref{eqn:adam_svtraj} and \cref{eqn:adam_w} with our penalty. See \cref{sec:math} for details.
\end{proof}
\vspace{-1mm}
Empirically, we note a new development: there is a large degree of \textit{depth invariance} as rank is pushed down and low test error is achieved almost independently of depth (\cref{fig:gd+r}, bottom row), even at depth 2 (i.e., a shallow network). Exact rank recovery of the underlying matrix is now possible at all depths, unlike gradient descent, and the networks converge to solutions faster by an order of magnitude. From a shallow network ($N=2$), increasing the depth does not induce any material changes in the solutions produced as the penalized DLNN under Adam produces low-rank solutions that achieve exact rank recovery and low test error faster and better than previous settings. 

Moreover, we see that this combination of Adam with the penalty also retains some of the accelerative effect of Adam. Specifically, we see more favorable generalization properties and smoother development of singular values whose convergence speed is at least an order of magnitude faster than under gradient descent ($10^4 \ \textrm{vs.} \ 10^5$ iterations)---whose effects do not vary much with depth. As the singular values evolve, we see that their paths are relatively robust to depth, exhibiting near-identical behavior with significantly less dependence on depth than before (\cref{fig:gd+r}, right most column).


\subsection{No depth, with penalty}

Given the beneficial depth-invariant effects produced by the penalty under Adam in both deep ($N>2$) and shallow ($N=2$) networks, we now consider its effects in a limiting case: a degenerate depth 1 network (i.e., no depth; $N=1$). It is known \cite{Zhang:2016ve} that gradient descent in this setting converges to the minimum Frobenius ($\ell_2$) norm solution, which does not necessarily induce low-rankedness. As expected, training a depth 1 network with gradient descent fails to produce a well-generalizing or low-rank solution (\cref{fig:lm}, top row) as learning is heavily constrained without depth.

\begin{figure}[htbp!]
    \centering
    \includegraphics[width=.9\linewidth]{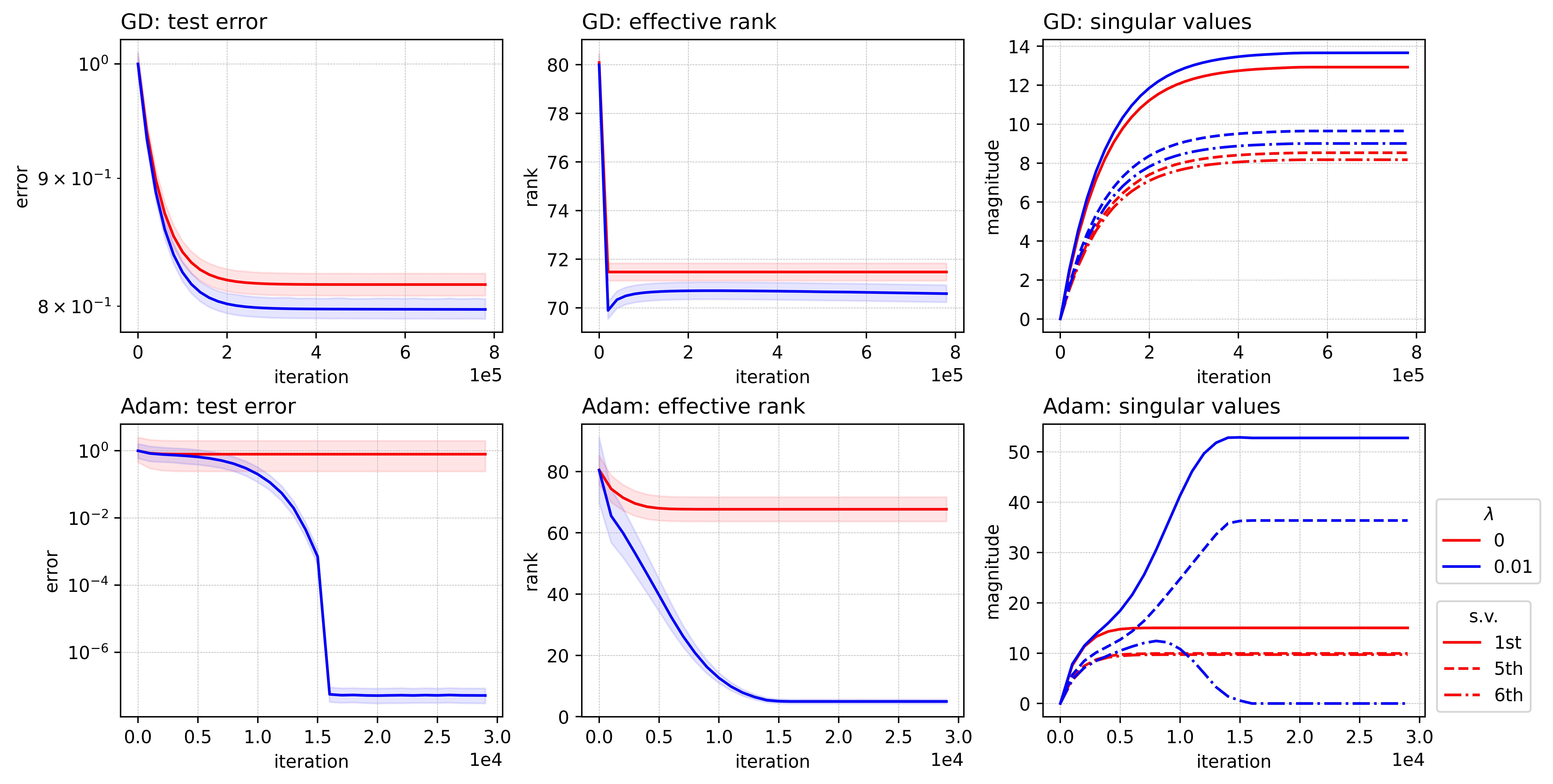}
    \caption{Performance comparison between choice of optimization and regularizer in a depth 1 network. Top row corresponds to gradient descent and bottom corresponds to Adam. Note that $\lambda = 0$ (red line) corresponds to the un-regularized setting. Here, $\lambda = 10^{-2}$ but results hold for values $\lambda \in [10^{-6},10^{-1}]$.} 
    \label{fig:lm}
\end{figure}

Yet, despite being ineffective under gradient descent, the penalty is again effective under Adam (\cref{fig:lm}, bottom row) even without depth, generalizing as well as if not better than deep networks ($N > 2$). We note that replacing our penalty with other proxies of rank or spectral sparsity, like the nuclear norm, does not work (\cref{tbl:diff_optsregs}). As described earlier, under un-regularized gradient flow with no depth, network dynamics collapse as singular value trajectories reduce to $\vb{u}_i^\top \nabla_W \mathcal{L}(W(t)) \vb{v}_i$ and the depth dependent accelerative pre-conditioning vanishes ($P_W = I_{mn}$). We see this empirically (e.g. \cref{fig:lm} top row and \cref{tbl:diff_optsregs}) as solutions from un-regularized gradient descent generalize poorly and fail at rank recovery. In contrast, a depth 1 network trained under Adam with the penalty not only achieves low test error (\cref{fig:lm}, bottom-left), but also recovers the underlying rank of the ground truth matrix---behaving qualitatively like a deep network. As such, we see that the depth invariance of Adam and the penalty in deeper networks also extends to the case of a depth-less degenerate network.

Without depth, a key component that appears to help the network learn under Adam and the penalty is its variance-weighted gradient term $\nabla_W \mathcal{L}(W) \cdot G$, as defined in \cref{eqn:adam_w}, along with the term $P_{W,G}\left(\lambda \frac{\mathrm{vec}(UV^\top - U \tilde{\Sigma} V^\top)}{\vert \vert W \vert \vert_F^2}\right)$ which reduces to $G\left(\lambda \frac{\mathrm{vec}(UV^\top - U \tilde{\Sigma} V^\top)}{\vert \vert W \vert \vert_F^2}\right)$ without depth. Interestingly, the variance of the loss gradient and the ratio $\eta^2 = \mathrm{var}(\nabla_W \mathcal{L}(W))/\nabla_W \mathcal{L}(W)^2$ formed from $\nabla_W \mathcal{L}(W) \cdot G$ resembles an inverse signal-to-noise ratio that both have come up in other works as important quantities that are strongly predictive of generalization capabilities \cite{fantastic} or are essential to finding optimal variance adaption factors for loss gradients \cite{Adam:sgn}.

\subsection{Comparison with other penalties and optimizers}
\begin{table}[h!]
\begin{center}
\scriptsize
\caption{
  Results for rank 5 matrix completion across various optimizer/penalty/depth combinations in terms of test error (\textbf{Err}) and effective rank (\textbf{Rk}, rounded to nearest integer) of the estimated matrix. \textbf{Ratio} denotes our ratio penalty ($||\cdot||_*/||\cdot||_F$), \textbf{Sch p:q} denotes the ratio of two Schatten (quasi)norms ($||\cdot||_{S_p}/||\cdot||_{S_q}$) as penalty, \textbf{Nuc} denotes the nuclear norm penalty, \textbf{None} is no penalty, and $a \cdot e$ $b$ denotes $a \cdot 10^{b}$. Best results---in terms of both absolute test error (lower the better) and rank reduction (closer to 5 the better) as well as depth invariance in terms of error and rank---are in bold. For more results, see \cref{sec:different_opt}.
  }
\begin{tabular}
{l@ {\qquad} c@ {\qquad} cc cc cc cc cc cc}
  \toprule
  \multirow{2}{*}{\raisebox{-\heavyrulewidth}{{Optimizer}}} & \multirow{2}{*}{\raisebox{-\heavyrulewidth}{{Depth}}} & \multicolumn{2}{c}{{Ratio}} & \multicolumn{2}{c}{{Sch $\frac{1}{2}$:$\frac{2}{3}$}}  & \multicolumn{2}{c}{{Sch $\frac{1}{3}$:$\frac{2}{3}$}}  & \multicolumn{2}{c}{{Sch $\frac{1}{3}$:$\frac{1}{2}$}} & \multicolumn{2}{c}{{Nuc}} & \multicolumn{2}{c}{{None}} \\
  \cmidrule{3-14}
  &  & {Err} & {Rk} & {Err} & {Rk} & {Err} & {Rk} & {Err} & {Rk} & {Err} & {Rk} & {Err} & {Rk}\\
  
  \toprule
  \multirow{2}{*}{Adam \cite{Kingma:2015us}} & 1     &  \textbf{4$\bm{e}\bf{\texttt{-}7}$} & \bf{5}  & 0.72 & 33  & 0.80 & 45  & 0.81 & 53   & 0.36 & 6  & 1.00 & 79 \\
                        & 3     & \textbf{4$\bm{e}\bf{\texttt{-}7}$} & \bf{5}  & 3$e\texttt{-}6$ & 5   & 1$e\texttt{-}5$ & 5  & 6$e\texttt{-}6$ & 5   & 0.30 & 5  & 0.04 & 6\\
  \midrule
  \multirow{2}{*}{Adagrad \cite{duchi2011adaptive}} & 1     & 0.58 & 31  & 0.81 & 60  & 0.97 & 32 & 0.79 & 60 & 0.12 & 8 & 0.80 & 70 \\
                           & 3     & 3$e\texttt{-}7$ & 5  & 9$e\texttt{-}7$ & 5 & 1$e\texttt{-}5$ & 5 & 2$e\texttt{-}7$ & 5 & 0.05 & 6 & 4$e\texttt{-}3$ & 6 \\
  \midrule
  \multirow{2}{*}{Adamax \cite{Kingma:2015us}}  & 1 & \textbf{4$\bm{e}\bf{\texttt{-}7}$} & \bf{5} & 0.76 & 44 & 0.85 & 22 & 0.80 & 58 & 0.05 & 6 & 0.81 & 72\\
                           & 3     & \textbf{7$\bm{e}\bf{\texttt{-}7}$} & \bf{5} & 3$e\texttt{-}6$ & 5 & 7$e\texttt{+}5$ & 1 & 6$e\texttt{-}6$ & 5 & 0.07 & 7 & 0.01 & 7\\
  \midrule
  \multirow{2}{*}{RMSProp} & 1     & 2$e\texttt{-}4$ & 6 & 0.08 & 4 & 1.6$e\texttt{+}3$ & 5 & 1.8$e\texttt{+}3$ & 8   & 0.05 & 8 & 0.80 & 70\\
                           & 3     & 0.03 & 5 & 8$e\texttt{-}4$ & 5 & 2$e\texttt{-}3$ & 5 & 1.9 & 5    & 0.05 & 6 & 0.11 & 14\\  
  \midrule
  \multirow{2}{*}{GD}  & 1     & 0.81 & 67 & 0.81 & 62 & 0.80 & 47 & 0.81 & 60    & 0.82 & 59 & 0.83 & 72\\
                            & 3     & 0.51 & 3 & 0.25 & 5 & 0.56 & 3 & 0.39 & 5   & 0.24 & 4 & 1$e\texttt{-}5$ & 5\\  
  \bottomrule
 \end{tabular}
  \label{tbl:diff_optsregs}
  \end{center}
\end{table}
\paragraph{Other Combinations} Our results do not preclude the possibility that other optimizers and penalties can produce similar or better effects. For completeness, and inspired by the properties of our penalty (e.g. ratio-like, spectral, non-convex), we experiment with various optimizers and penalties (\cref{tbl:diff_optsregs}) to compare their interactions across shallow ($N=1$) and deep ($N=3$) settings. We note that our ratio penalty under Adam and its close variant Adamax, both initially proposed in \cite{Kingma:2015us}, are the only combinations that largely show depth invariance in test error and rank recovery whereas others require depth to reduce rank or generalize better. Though the nuclear norm exhibits some depth invariance, it is unable to find well-generalizing solutions and fails in rank reduction under gradient descent where depth is still necessary. Other combinations also fail in enabling a depth 1 network to perform as well as deeper ones. Due to space constraints, we leave a fuller treatment and characterization of each optimizer's inductive bias and its interaction with different regularizers for future work.

\begin{figure}[h!]
  \centering
  \includegraphics[width=1\linewidth]{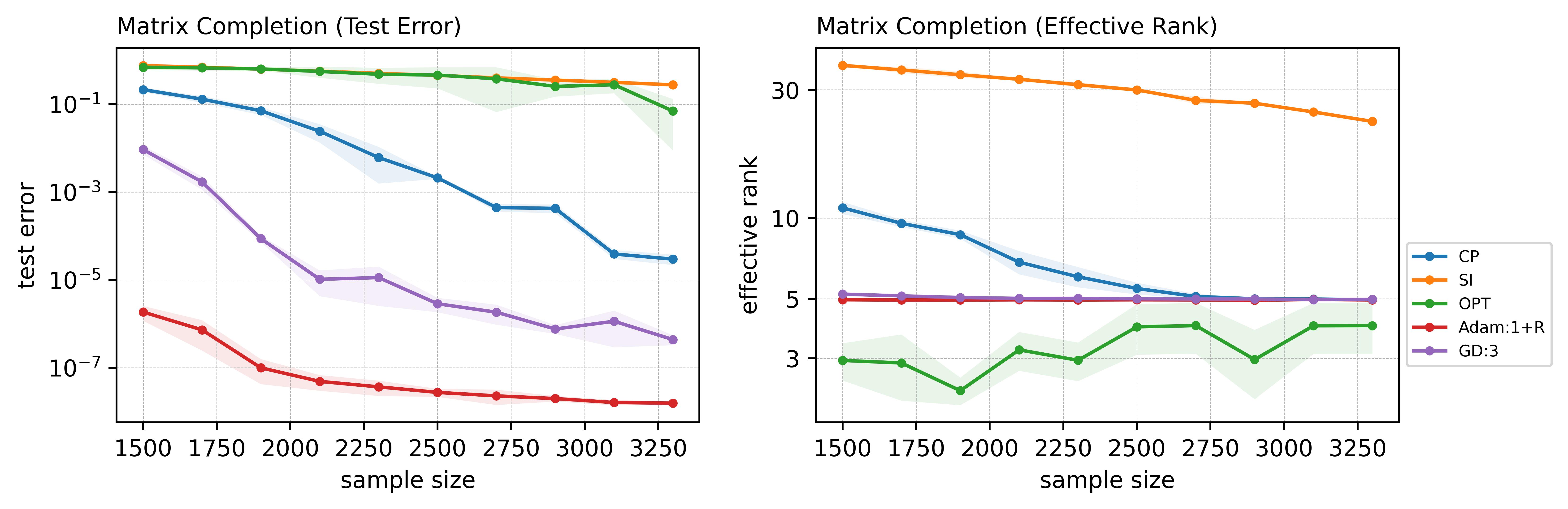}
  \caption{Comparative performance in test error and rank minimization for rank 5 matrix completion. $x$-axis stands for the number of observed entries ($\mathbb{R}^{100 \times 100}$, so out of $100 \times 100 = 10^4$ entries) and shaded regions indicate error bands. \textbf{Adam:1+R} refers to a depth 1 network trained with Adam and penalty, \textbf{CP} is the minimum nuclear norm solution, \textbf{GD:3} is a depth 3 network trained with gradient descent, \textbf{OPT} is \texttt{OptSpace} \cite{keshavan2009gradient}, and \textbf{SI} is \texttt{SoftImpute} \cite{softimpute2010}. To reduce clutter, we omit results with similar performance as \textbf{GD:3} (e.g. GD:4, GD:5). 
  }
  \label{fig:combo}
\end{figure}

\paragraph{Comparative Performance} We also note our penalty's comparative performance (\cref{fig:combo}) against other methodologies for matrix completion across various sample sizes (i.e., the amount of observed entries, uniformly sampled, made available for training). A depth 1 network with Adam and the penalty (\cref{fig:combo}, \textbf{Adam:1+R}, red line) outperforms all other methods including an un-regularized DLNN in terms of test error and degree of rank compression/recovery across varying sample sizes. Even at lower sample sizes, the depth 1 network generalizes better than methods such as \texttt{SoftImpute} \cite{softimpute2010}, \texttt{OptSpace} \cite{keshavan2009gradient}, the minimum nuclear norm solution \cite{candes2010power}, and DLNNs trained with gradient descent ($N\geq3$) by at least an order of magnitude. It also outperforms other methods across various data regimes from small sample sizes to large ones, improving as the sample size grows.

\section{Results on real data}
\label{sec:empirical}

\begin{table}[h]
\scriptsize
  \caption{Performance evaluations on MovieLens100K. Results are categorized by model, whether additional data or features (e.g. external/side information, implicit feedback, graph features, etc.) beyond the explicit ratings in the interaction matrix are used, and test error as measured by root mean squared error (RMSE, lower is better) on different splits in \textbf{(a)} and \textbf{(b)}. Since various approaches use different train-test proportions, results \cite{graphrec,2018GraphConvMatComp} on two common splits are included. Results from using Adam with the penalty are in bold.}
\subfloat[Performance on 90:10 (90\%) train-test split]
{\begin{tabular}{l@{\qquad}c@{\qquad}r}
  \toprule
  \multirow{5}{*}{\raisebox{-\heavyrulewidth}{Model}} &
  \multirow{2}{*}{{Uses side info,}} \\ &
  \multirow{2}{*}{{add. features, or}} \\ & 
  \multirow{2}{*}{{other info, etc?}} &   
  \multicolumn{1}{c}{$90\%$} \\
  \cmidrule{3-3} 
  &  & \hspace{1.5mm} RMSE \\
  \toprule
  {Depth 1 LNN}                          &  No \\
  \hspace{3mm} w. GD                     &  &  2.814  \\
  \hspace{3mm} w. GD+penalty             &  &  2.808  \\
  \hspace{3mm} w. Adam                   &  &  1.844  \\
  \hspace{3mm} \textbf{w. Adam+penalty}   &  & \textbf{0.915}  \\
  \midrule
  {User-Item Embedding}                 &  No \\
  \hspace{3mm} w. GD                    &  & 2.453 \\
  \hspace{3mm} w. GD+penalty            &  & 2.535 \\
  \hspace{3mm} w. Adam                  &  & 1.282 \\
  \hspace{3mm} \textbf{w. Adam+penalty}  &  & \textbf{0.906} \\
\midrule
  {NMF \cite{pmf}}                      &  No  & 0.958  \\
  {PMF \cite{pmf}}                      &  No  & 0.952  \\
  {SVD++ \cite{svd}}                    &  Yes  & 0.913 \\
  {NFM \cite{nfm}}                      &  No  & 0.910 \\  
  {FM \cite{fm}}                        &  No  & 0.909 \\
  {GraphRec  \cite{graphrec}}           &  No  & 0.898 \\
{AutoSVD++ \cite{autosvd2}}             &  Yes  & 0.904 \\
  {GraphRec+sidefeat.\cite{graphrec}} &  Yes  & 0.899 \\
{GraphRec+graph/side feat.\cite{graphrec}}  &  Yes  & 0.883 \\
  \bottomrule
 \end{tabular}
}
 \quad
\subfloat[Performance on 80:20 (80\%) train-test split]
{
\begin{tabular}{l@{\qquad}c@{\qquad}r}
  \toprule
  \multirow{5}{*}{\raisebox{-\heavyrulewidth}{Model}} & \multirow{2}{*}{{Uses side info,}} \\ &
  \multirow{2}{*}{{add. features, or}} \\ &
\multirow{2}{*}{{other info, etc?}} &  
  \multicolumn{1}{c}{$80\%$}  \\
  \cmidrule{3-3} 
&  &  \hspace{1.5mm}RMSE \\
  \toprule
  {Depth 1 LNN}                           & No \\
  \hspace{3mm} w. GD                      &     & 2.797  \\
  \hspace{3mm} w. GD+penalty              &     & 2.821   \\
  \hspace{3mm} w. Adam                    &     & 1.822  \\
  \hspace{3mm} \textbf{w. Adam+penalty}   &     & \textbf{0.921}  \\
  \midrule
  {User-Item Embedding}  &  No \\
  \hspace{3mm} w. GD           &     & 2.532  \\
  \hspace{3mm} w. GD+penalty   &     & 2.519  \\
  \hspace{3mm} w. Adam         &     & 1.348  \\
  \hspace{3mm} \textbf{w. Adam+penalty}   &  & \textbf{0.919}  \\
\midrule
\\
  \hspace{.5mm}{IMC \cite{imc1,imc2}}            &  Yes   & 1.653  \\
  \hspace{.5mm}{GMC \cite{Bronstein2014matCompGraph}}                         &  Yes  & 0.996  \\
  \hspace{.5mm}{MC \cite{candes2009exact}}                              &  Yes   & 0.973  \\
  \hspace{.5mm}{GRALS \cite{Rao2015collabFilterGraph}}                           &  Yes   & 0.945  \\
  \hspace{.5mm}{sRGCNN (sRMGCNN) \cite{Bronstin2017GeoMatComp}}      &  Yes   & 0.929  \\
  \hspace{.5mm}{GC-MC \cite{2018GraphConvMatComp}}                           &  Yes  & 0.910  \\
  \hspace{.5mm}{GC-MC+side feat. \cite{2018GraphConvMatComp}}                &  Yes  & 0.905  \\
 \\
  \bottomrule
 \end{tabular}
 }
  \label{tbl:mlens100k_results}
\end{table}
Lastly, a natural question might be: how well do our results extend to real-world data? To answer this, we consider MovieLens100K \cite{movieLensData}---a common benchmark used to evaluate different approaches for recommendation systems. It consists of ratings from 944 users on 1,683 movies, forming an interaction matrix $M \in \mathbb{R}^{944 \times 1683}$ where the goal is to predict the rest of $M$ after observing a subset.

Unlike our earlier experiments, the values here are discrete in the range $\{1,2,3,4,5\}$ and $M$ is of high, or near full, rank. Given these differences and more conventional empirical approaches in recommendation systems, we apply our penalty in two ways. The first way is as before: training a depth 1 network with Adam and the penalty (Depth 1 LNN, \cref{tbl:mlens100k_results}). The second way is to impose our penalty on a classic user-item embedding model (User-Item Embedding, \cref{tbl:mlens100k_results} \cite{explicitmodel}) that combines user-specific and item-specific biases with a dot product between a latent user and latent item embedding; we apply our penalty separately on either the item or the user embedding layer. Though approaches solely utilizing explicit ratings have fallen out of favor in lieu of ones incorporating additional information and complex designs (e.g. graph-based, implicit feedback, deep non-linear networks), we nonetheless evaluate the effects of our penalty within this simple framework. We compare the results from these two approaches with a variety of approaches that use specialized architectures, deep non-linear networks, additional side information, etc., beyond $M$.

From \cref{tbl:mlens100k_results}, we see that Adam and the penalty (\textbf{w. Adam+penalty}) can improve performance over the baseline of gradient descent (GD) or Adam alone. Surprisingly, a depth 1 network with Adam and the penalty can outperform or come close to other more specialized approaches despite its simplicity; however, in contrast to the other methods, it does so without any specialized or additional architectures (e.g. helper models/networks), external information beyond $M$ (e.g. implicit feedback, side information), construction of new graphs or features, non-linearites, higher-order interactions (e.g. factorization machines), or---for the depth 1 case---even any factorization at all. More precise tuning (e.g. better initialization, learning schedules) or usage of other information/features may yield further improvements on this task and others that involve matrix completion or factorization \cite{2020MEQLearn, 2019MEnet, 2019MatFactNLPEmbed}. We leave these for fuller evaluation and further study in future work.

\section{Discussion}
\label{sec:discuss}
The dynamics of optimization trajectories---induced together by the choice of optimizer, parameterization, loss function, and architecture---can play an important role in the solutions produced and their ability to generalize. Depth and gradient descent-like algorithms have been key ingredients to the successes of deep learning. On matrix completion/factorization, the proposed penalty helps produce well-generalizing solutions and perfect rank recovery even with a degenerate depth 1, or depth-less, network. Does that mean our penalty, together with Adam's own inductive bias, is producing an effect similar to implicit regularization under gradient descent with depth, but better? 


We suspect not. While we concur with the conjecture in \cite{Arora:2019ug}---namely, a reductionist view which suggests that implicit regularization can be entirely encapsulated by an explicit norm-based penalty is likely an incorrect over-simplification---we believe that there is merit in studying both implicit and explicit forms of regularization to examine their interplay. Our work suggests that we may be able to partially replicate the successes of \textit{deep} learning by selectively combining optimization methods with explicit penalties via better model design or encoding of inductive biases, but this remains unclear.  

Many questions remain open from our limited analyses which, due to space considerations, we leave for future work. For instance, how well do lessons from DLNNs translate to their non-linear counterparts or other tasks (e.g. classification)? How does this relate to learning regimes with larger learning rates or discrete trajectories (i.e., beyond gradient flow)? A more rigorous analysis of the properties (e.g. convergence) of Adam, adaptive gradient methods, and other optimizers in the presence of explicit regularizers may better our understanding. It remains unclear whether implicit regularization is a bug or a feature in deep over-parameterized networks. Nonetheless, our findings suggest the possibility that it can be harnessed and transformed to desirable effect. 

\acksection
We thank Kellin Pelrine for his feedback and Derek Feng for his help with a previous collaboration on which this work builds. We also thank the anonymous reviewers for their valuable comments and feedback throughout the process.

\bibliographystyle{plainnat}
\bibliography{bib, custom}

\section*{Checklist}

\begin{enumerate}

\item For all authors...
\begin{enumerate}
  \item Do the main claims made in the abstract and introduction accurately reflect the paper's contributions and scope?
    \answerYes{}
  \item Did you describe the limitations of your work?
    \answerYes{See \cref{sec:findings}, \cref{sec:discuss}, and Appendix.}
  \item Did you discuss any potential negative societal impacts of your work?
    \answerNA{Our work does not present any foreseeable societal impacts or consequences.}
  \item Have you read the ethics review guidelines and ensured that your paper conforms to them?
    \answerYes{}
\end{enumerate}

\item If you are including theoretical results...
\begin{enumerate}
  \item Did you state the full set of assumptions of all theoretical results?
    \answerYes{See \cref{sec:findings} and Appendix.}
        \item Did you include complete proofs of all theoretical results?
    \answerYes{See \cref{sec:findings} and Appendix.}
\end{enumerate}

\item If you ran experiments...
\begin{enumerate}
  \item Did you include the code, data, and instructions needed to reproduce the main experimental results (either in the supplemental material or as a URL)?
    \answerYes{See supplementary material.} 
  \item Did you specify all the training details (e.g., data splits, hyperparameters, how they were chosen)?
    \answerYes{See Appendix.}
        \item Did you report error bars (e.g., with respect to the random seed after running experiments multiple times)?
    \answerNA{Results under our proposed penalty are valid even across varying degrees of regularization strength ($\lambda$) and several initial learning rates $\alpha$ and across several initializations. Other results that were run with random seeds had negligible standard errors (see implementation details under the Appendix). 
    }
        \item Did you include the total amount of compute and the type of resources used (e.g., type of GPUs, internal cluster, or cloud provider)?
    \answerYes{See Appendix.}
\end{enumerate}

\item If you are using existing assets (e.g., code, data, models) or curating/releasing new assets...
\begin{enumerate}
  \item If your work uses existing assets, did you cite the creators?
    \answerNA{}
  \item Did you mention the license of the assets?
    \answerNA{}
  \item Did you include any new assets either in the supplemental material or as a URL?
    \answerNA{}
  \item Did you discuss whether and how consent was obtained from people whose data you're using/curating?
    \answerNA{}
  \item Did you discuss whether the data you are using/curating contains personally identifiable information or offensive content?
    \answerNA{}
\end{enumerate}

\item If you used crowdsourcing or conducted research with human subjects...
\begin{enumerate}
  \item Did you include the full text of instructions given to participants and screenshots, if applicable?
    \answerNA{}
  \item Did you describe any potential participant risks, with links to Institutional Review Board (IRB) approvals, if applicable?
    \answerNA{}
  \item Did you include the estimated hourly wage paid to participants and the total amount spent on participant compensation?
    \answerNA{}
\end{enumerate}

\end{enumerate}

\clearpage

\appendix
\section{Implementation, Further Experiments, \& Derivations}
\label{sec:implementation+experiments}
\vspace{-4mm}
\subsection{Implementation}
\vspace{-2mm}
\label{sec:implementation_details}
In this section, we provide the implementation details behind our experiments. We use \texttt{PyTorch} for implementing the linear neural network, \texttt{CVXPY} for finding the minimum nuclear norm solution, the \texttt{R} package \texttt{ROptSpace} for running the \texttt{OptSpace} algorithm, and the \texttt{Python} package \texttt{fancyImpute} for running the \texttt{SoftImpute} algorithm. Additionally, for experiments on MovieLens100k, we use the \texttt{Spotlight} repo \footnote{(\href{https://github.com/maciejkula/spotlight}{https://github.com/maciejkula/spotlight})} for different kinds of recommendation models in \texttt{PyTorch} and modify its common building blocks for our purposes. Experiments were run on a NVIDIA Tesla V100 GPU.

\paragraph{Synthetic Data} For synthetic data experiments, our implementation details mirror the experimental design of \cite{Arora:2019ug}, which we briefly detail. When referring to a random rank $r$ matrix with size $m \times n$, we mean a product $UV^\top$, where the entries of $U \in \mathbb{R}^{m \times r}$ and $V \in \mathbb{R}^{n \times r}$ are drawn independently from a standard normal distribution. The observed entries are uniformly sampled at random without repetition. The goal is to recover the underlying matrix by observing a portion of it and trying to recover the remaining entries as well as producing a solution matrix whose overall rank is low or is close to the rank of the underlying ground truth matrix. During training, deep linear neural networks are trained by gradient descent, or whichever other optimizer under consideration, under the Frobenius loss (i.e. $\ell_2$ loss).

Weights are initialized via independent samples from a Gaussian distribution with zero mean and standard deviation $10^{-3}$ as in \cite{Arora:2019ug}. Learning rates are fixed through training; in line with previous work \cite{Arora:2019ug} and the assumptions of gradient flow, we set our initial learning rates (for gradient descent, Adam, and our other optimizers) $\alpha = 10^{-3}$ for the results shown in this paper, but we also conducted the same experiments with $\alpha \in \{5 \cdot 10^{-4}, 10^{-4}\}$ and saw no qualitative differences. During training, stopping criteria consist of either reaching a total of $5 \cdot 10^5$ iterations or training loss reaching values below $10^{-7}$. Results relating to the ratio penalty are robust across varying strengths of regularization $\lambda \in [10^{-4},10^{-1}]$.
For Adam, aside from the initial learning rate, the default hyper-parameters under PyTorch's implementation were used, such as $(\beta_1, \beta_2) = (0.9, 0.999)$, no \texttt{amsgrad}, etc. unless otherwise stated. Similarly for experiments with other optimizers, aside from the initial learning rate, we use the default hyper-parameters unless otherwise specified.
Test error is measured via mean-squared error, unless stated otherwise. To quantify the rank of solutions produced, we use the effective rank measure \cite{roy2007effective} due to its favorable properties over numerical rank in simulation settings and its usage in \cite{Arora:2019ug}. The test error with respect to a ground truth matrix $W^\star$ is given by the mean-squared error (i.e., $\frac{1}{N}\norm{W - W^\star}_{F}^2$) but results continue to hold under similar measures of test error. For rank 5 experiments where the sample size was fixed, sample size is set at 2000 (again, following \cite{Arora:2019ug}) but results are similar for sample sizes at 2500, 3000, etc. The sample size for rank 10 experiments were set at 3000 and 3500.

\paragraph{MovieLens100k} For our experiments on MovieLens100k, we use the base explicit factorization model available in \texttt{Spotlight} in conjunction with Adam and our ratio penalty. The explicit factorization model resembles a shallow factorization but includes several notable differences: an user-embedding layer, an item-embedding layer, an user-specific bias, and an item-specific bias. With our penalty applied to either the user-embedding layer or the item-embedding layer, we evaluate values of $\lambda \in [10^{-6}, 10^{-1}]$, embedding dimensions in $\{2,4,8,16,32,64,128\}$, and batch-sizes of $\{128,256,512,1024\}$
for initial learning rates $\alpha \in [5 \cdot 10^{-3},  10^{-5}]$. Each model configuration is trained for 100 epochs and evaluated on the full test set at the end. The best model is check-pointed and training is terminated if training error begins increasing for three epochs. The best performing model with the reported test error was with learning rate $\alpha = 10^{-3}$, batch size 256, effective dimension 64, $\lambda_{\rm{user}} = 0.01$ while another was $\alpha = 10^{-3}$, batch size 128, effective dimension 64, $\lambda_{\rm{user}} = 10^{-4}$ and $\lambda_{\rm{item}} = 10^{-1}$. For the depth 1 LNN, we take an approach near-identical to our synthetic experiments but include a bias term that is added to our single weight matrix. We explore values for regularization strengths $\lambda \in [2\cdot 10^{1}, 10^{-6}]$ and initial learning rate $\alpha \in [5 \cdot 10^{-3},  10^{-5}]$. The best performing model with the reported test error was with learning rate $\alpha = 5 \cdot 10^{-4}$ and $\lambda = 1.5$ though many similar configurations came close to the same performance. A more rigorous sweep of configurations is necessary to exhaustively find similar, or the best, configurations; due to space considerations, we leave more in-depth evaluations for future work.


\subsection{Rank 10 Matrix Completion}
\label{sec:rank10}

The following details the same set of experiments as those in \cref{sec:findings} but for rank 10 matrices as an example to illustrate that our results generalize beyond rank 5.
\begin{figure}[h!]
  \centering
  \includegraphics[width=1\linewidth]{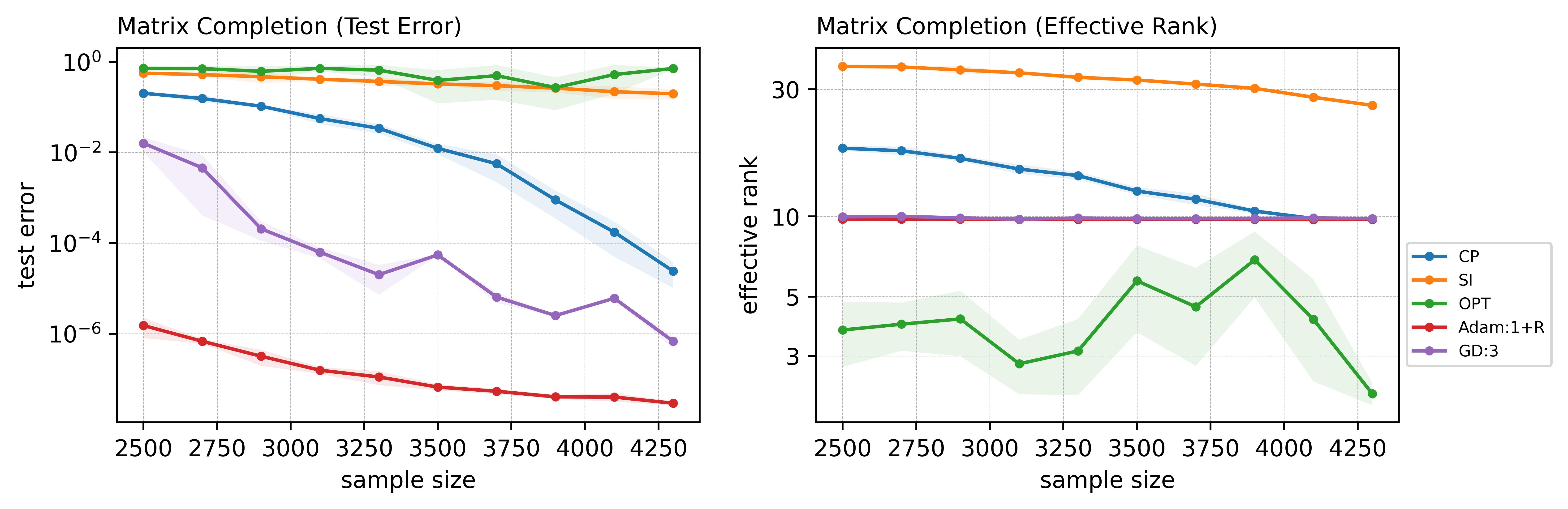}

  \caption{Comparative performance in generalization error and rank minimization for rank 10 matrix completion (of size $100 \times 100$).
  $x$-axis stands for the number of observed entries (uniformly sampled) and shaded regions indicate standard error bands. \textbf{Adam:1+R} refers to Adam at depth 1 with the penalty, \textbf{CP} is the minimum nuc. norm solution, \textbf{GD:3} is gradient descent with depth 3 (deep matrix factorization), \textbf{OPT} is \texttt{OptSpace} \cite{keshavan2009gradient}, and \textbf{SI} is \texttt{SoftImpute} \cite{softimpute2010}. To reduce clutter, we omit results with similar performance to GD:3 (e.g. GD:4/5 etc.).
  }
  \label{fig:combo_r10}
\end{figure}

Overall, comparative performance results very much resemble those in \cref{fig:combo} for rank 5 completion. As seen in \cref{fig:combo_r10}, our penalty's comparative performance relative to other techniques remains unchanged from rank 5. Like before, Adam with penalty, at depth 1, outperforms all other approaches across all data regimes by at least two orders of magnitude, with reconstruction error decreasing even further as sample size is increased. A similar result is seen in terms of its performance on rank: it reduces rank to the point of exact rank recovery independently of sample size.

\begin{figure}[htbp!]
    \centering
    \includegraphics[width=.9\linewidth]{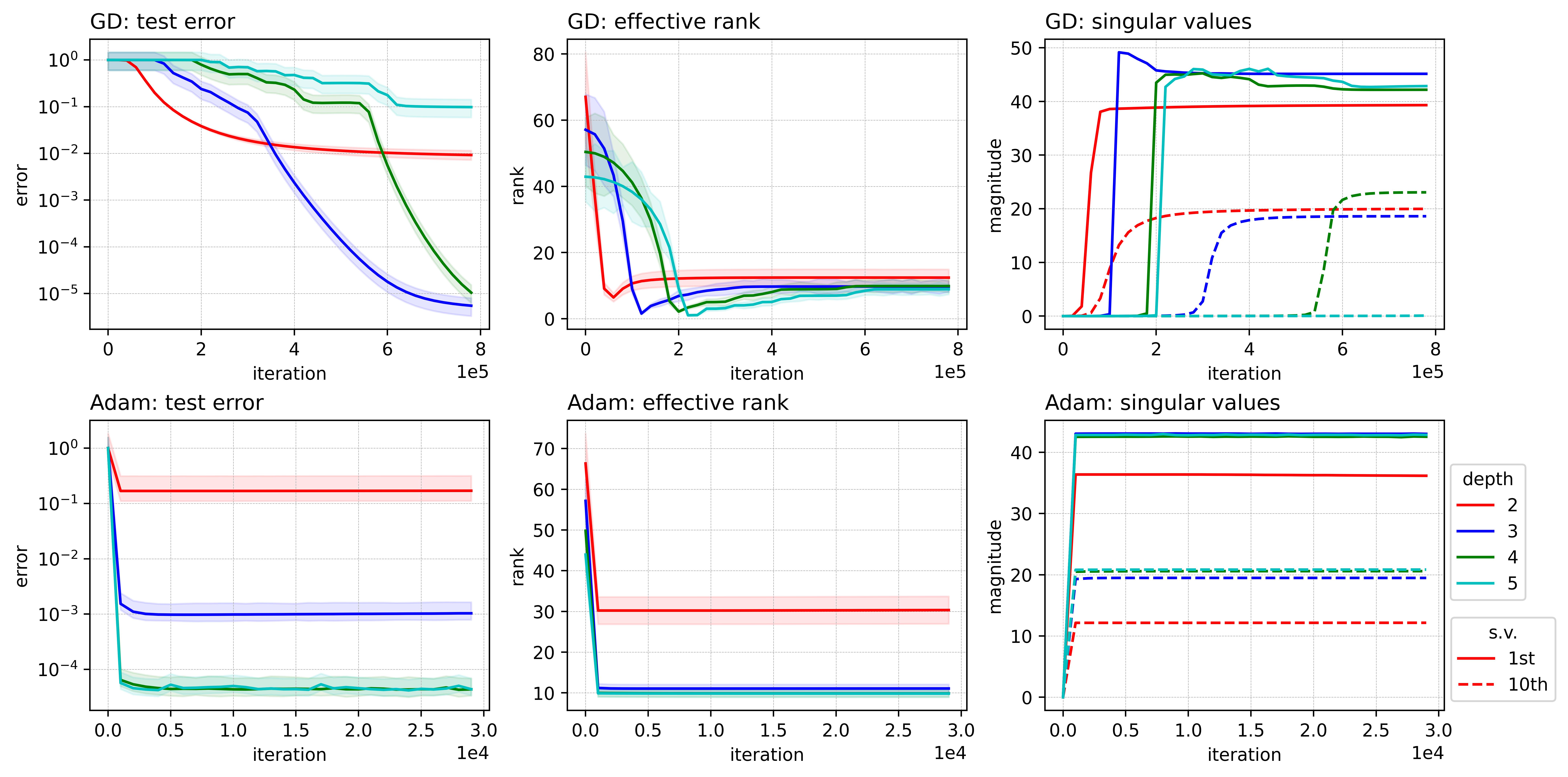}
    \caption{Dynamics of gradient descent and Adam over DLNN. Plots above show the performance of gradient descent over networks of depths 2/3/4/5 for rank 10 matrices of size $100 \times 100$. The top row depicts gradient descent and the bottom depicts Adam. The left column depicts generalization error as a function of depth and training iterations. The middle column depicts the change in effective rank across depths and over training iterations. The right column shows the 1st and 10th largest singular values for each depth across training iterations. For singular values, within each depth level (colored lines), a solid line indicates the 1\textsuperscript{st} largest singular value while a dotted line indicates the 10\textsuperscript{th} largest. We omit the other singular values for clarity due to their small magnitude.}
    \label{fig:gd_r10}
\end{figure}

Extending our experiments to rank 10 matrix completion, we see extreme similarity with  results under rank 5 (\cref{fig:gd_r10}). For singular values, we focus on the 1st and 10th largest values in accordance with the changed rank setting (rank 5 to rank 10) to better illustrate the differential behavior of the most prominent values. Our overall results remain the same as in \cref{sec:findings}.
\begin{figure}[h!]
    \centering
    \includegraphics[width=.9\linewidth]{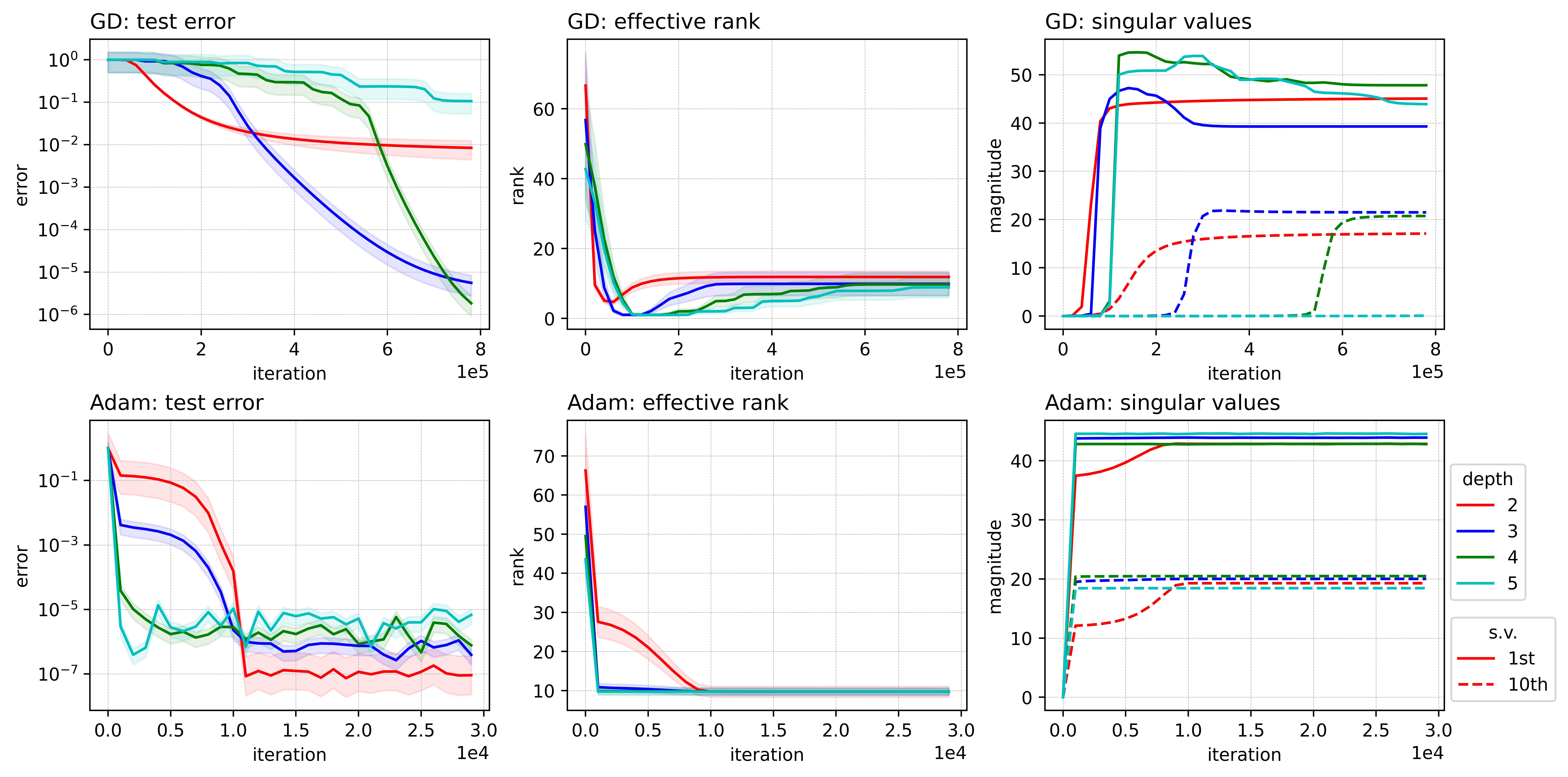}
    \caption{Dynamics of regularized gradient descent and regularized Adam with our penalty. Plots show the performance over networks of depths 2/3/4/5 for rank 10 matrix completion. The penalty's hyper-parameter is set at $\lambda = 10^{-4}$ but results hold for a range of values $\lambda \in [10^{-6},10^{-1}]$.
    }
    \label{fig:gd+r_r10} 
\end{figure}

For regularized gradient descent and Adam (\cref{fig:gd+r_r10}), results again resemble the rank 5 case. Under gradient descent, the proposed penalty induces only slight effects on quickening and tightening the convergence behavior of singular values and test error. In contrast, we see nearly identical effects of our penalty under Adam when compared to the rank 5 case. We also see that, like in rank 5, depth 2 becomes effective in generalizing and reducing rank when compared to depth 2 under Adam alone; in depth 2, our proposed penalty is able to push its singular value trajectories of the 1\textsuperscript{st} and 10\textsuperscript{th} largest singular values towards the other respective singular values in higher depth levels as if getting depth 2 to behave similarly as higher depth levels in contrast to un-penalized Adam.


\begin{figure}[h!]
    \centering
    \includegraphics[width=.9\linewidth]{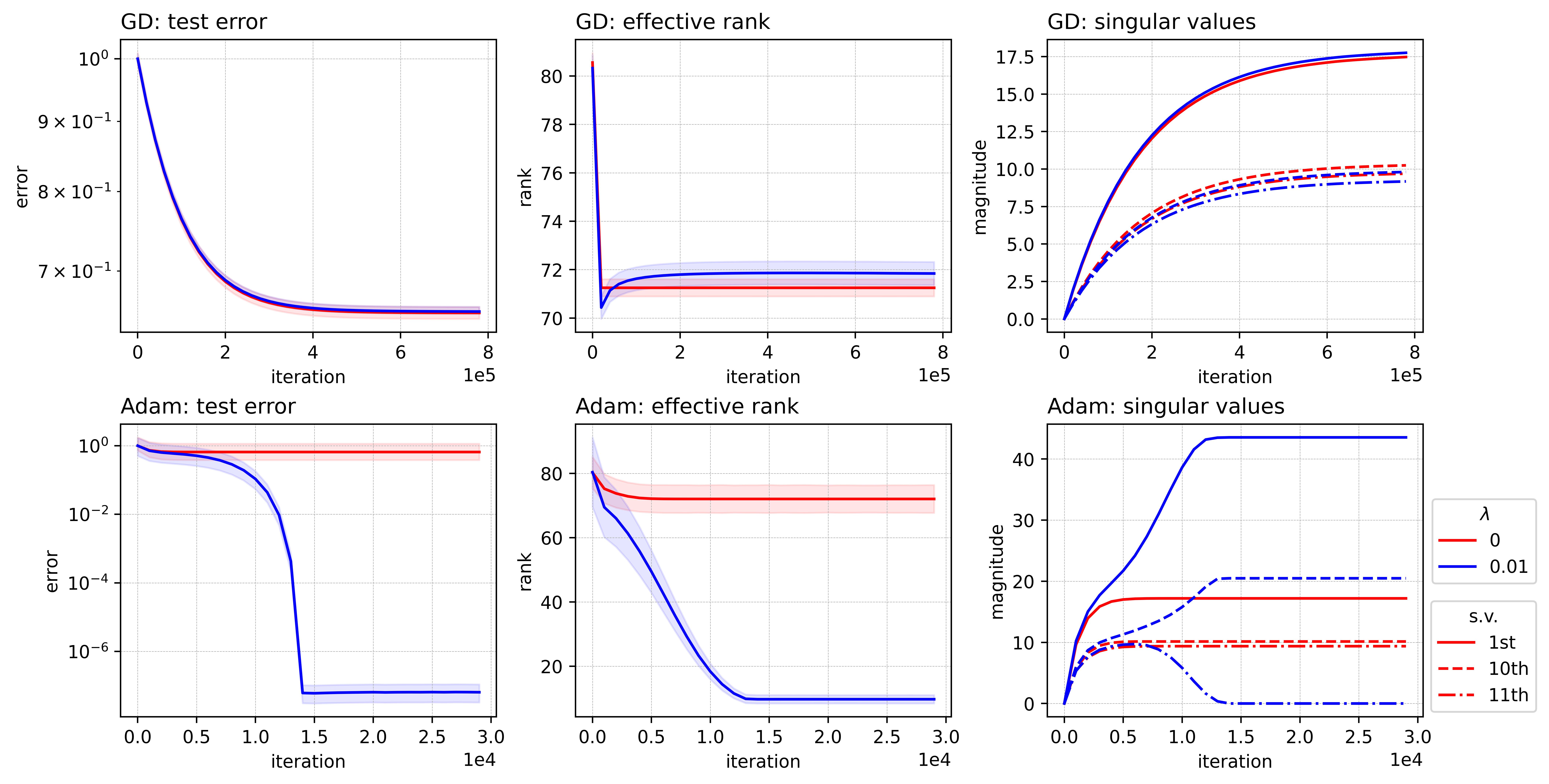}
    \caption{Performance comparison between choice of optimization and regularizer in a depth-1 linear network. Top row corresponds to gradient descent and bottom corresponds to Adam. Note that $\lambda = 0$ corresponds to the un-regularized setting.
    The penalty's hyper-parameter is set at $\lambda = 0.01$ but results hold for a range of values $\lambda \in [10^{-4},10^{-1}]$. For illustrative emphasis, the 11th largest singular value (dashed line) is also included here to depict the trajectory difference under penalized gradient descent versus penalized Adam for depth 1.}
    \label{fig:lm_r10}
\end{figure}

\subsection{Comparing Different Optimizers \& Penalties}
\label{sec:different_opt}

In this section, we expand upon our experiments evaluating different combinations of optimizers and explicit penalties beyond those in \cref{tbl:diff_optsregs}. In \cref{tbl:diff_optsregs_more1} and \cref{tbl:diff_optsregs_more2}, we include other optimizers and evaluate them against several explicit penalties. For all optimizer choices, hyper-parameters, aside from the initial learning rate which is set to $10^{-3}$, are set to their default values in $\texttt{PyTorch}$ unless otherwise specified (e.g. default momentum values for GD). For all experiments evaluating optimizer-penalty combinations, we try $\lambda \in [10^{-6}, 10^{-1}]$ which all delivered similar results, especially in terms of observed depth invariance (or lack thereof) and generalization/rank reduction performance. The results shown in \cref{tbl:diff_optsregs} in \cref{sec:different_opt} as well as in \cref{tbl:diff_optsregs_more1} and \cref{tbl:diff_optsregs_more2} here have $\lambda = 0.05$. Results shown are trained on a sample size of 2000 entries (out of a $100 \times 100$ matrix); results show no meaningful difference on other sample sizes (2500, 3000).

\begin{table}
\begin{center}
\scriptsize
\caption{%
  Additional results for rank 5 matrix completion across various optimizer/penalty/depth combinations in terms of test error (\textbf{Err}) and effective rank (\textbf{Rk}, rounded to nearest integer) of resulting solutions at convergence. \textbf{Ratio} denotes our ratio penalty ($||\cdot||_*/||\cdot||_F$), \textbf{Sch p:q} denotes the ratio of two Schatten (quasi)norms ($||\cdot||_{S_p}/||\cdot||_{S_q}$) as penalty, \textbf{Nuc} denotes nuclear norm penalty, and \textbf{None} denotes no penalty.
  }
\begin{tabular}{l@{\qquad}c@{\qquad} cc cc cc cc cc cc}
  \toprule
  \multirow{2}{*}{\raisebox{-\heavyrulewidth}{\textbf{Optimizer}}} & \multirow{2}{*}{\raisebox{-\heavyrulewidth}{\textbf{Depth}}} & \multicolumn{2}{c}{\textbf{Ratio}} & \multicolumn{2}{c}{\textbf{Sch $\frac{1}{2}$:$\frac{2}{3}$}}  & \multicolumn{2}{c}{\textbf{Sch $\frac{1}{3}$:$\frac{2}{3}$}}  & \multicolumn{2}{c}{\textbf{Sch $\frac{1}{3}$:$\frac{1}{2}$}} & \multicolumn{2}{c}{\textbf{Nuc}} & \multicolumn{2}{c}{\textbf{None}} \\
  \cmidrule{3-14}
  &  & \textbf{Err} & \textbf{Rk} & \textbf{Err} & \textbf{Rk} & \textbf{Err} & \textbf{Rk} & \textbf{Err} & \textbf{Rk} & \textbf{Err} & \textbf{Rk} & \textbf{Err} & \textbf{Rk}\\
  
  \toprule

Adam \cite{Kingma:2015us} (w. amsgrad)
& 1     & 0.83 & 29    & 0.99 & 53   & 1.00 & 61   & 1.00 & 60    & 0.34 & 6 
& 1.00 & 79\\
& 3     & 9$e\texttt{-}$3 & 5     & 0.05 & 6    & 0.04 & 6    & 0.10 & 6     & 0.32 & 6
& 0.06 & 6\\
\midrule     
  \multirow{2}{*}{Adadelta \cite{adadelta}}  
  & 1     & 0.98 & 59   & 1.00 & 63   & 0.98 & 37   & 1.00 & 53    & 0.77 & 21   & 1.01 & 74\\
& 3     & 3$e\texttt{-}$3 & 5     & 2$e\texttt{-}$4 & 5   & 3$e\texttt{-}$3 & 5   &   2$e\texttt{-}$3 & 5   & 0.29 & 5   & 5$e\texttt{-}$3 & 5\\                
\midrule
GD (w. momentum)
& 1   & 0.80 & 70    & 0.79 & 64   & 0.77 & 47   & 0.80 & 63 
      & 0.78 & 54   & 0.81 & 69\\
& 3   & 3$e\texttt{-}$3 & 5     & 1$e\texttt{-}$4 & 5    & 0.43 & 4    & 0.07 & 5
      & 0.02 & 5   & 0.07 & 4\\

\midrule
  \multirow{2}{*}{AdamW \cite{adamw}}  & 1     & 1$e\texttt{-}$3 & 5   & 0.98 & 45   & 1.00 & 49    
  & 1.00 & 64    & 0.34 & 6   & 1.00 & 79\\
                            & 3     & 1$e\texttt{-}$3 & 5   & 2$e\texttt{-}$4 & 5   & 0.01 & 5    
                            & 3$e\texttt{-}$2 & 5   & 0.31 & 6   & 0.02 & 5\\                                                          
\midrule                            
  \multirow{2}{*}{NAdam \cite{nadam}}  & 1     & 1$e\texttt{-}$3 & 5    & 0.97 & 43 &    0.99 & 54   
  & 0.99 & 61    & 0.32 & 6   & 1.01 & 79\\
                            & 3     & 2$e\texttt{-}$3 & 5    & 1$e\texttt{-}$4 & 5 &    0.05 & 5   
                            & 0.01 & 5   & 0.30 & 5   & 0.02 & 5\\                                     

\midrule
\multirow{2}{*}{RAdam \cite{radam}}  & 1   & 1$e\texttt{-}$3 & 5 &   0.95 & 46   & 0.99 & 51   & 1.00 & 62      & 0.25 & 6   & 1.00 & 80\\
                          & 3   & 7$e\texttt{-}$4 & 5 &   5$e\texttt{-}$4 & 5   & 0.01 & 5   & 0.01 & 5   
    & 0.18 & 5   & 0.05 & 6\\                                     
                            
  \bottomrule
 \end{tabular}
\label{tbl:diff_optsregs_more1}
\end{center}
\end{table}

\begin{table}[h]
\begin{center}
\scriptsize
 \caption{
  Additional results for rank 5 matrix completion across various optimizer/penalty/depth combinations in terms of test error (\textbf{Err}) and effective rank (\textbf{Rk}, rounded to nearest integer) of resulting solutions at convergence. \textbf{Sch p} denotes the Schatten (quasi)norm ($||\cdot||_{S_p}$) as penalty.
  }
\begin{tabular}{l@{\qquad}c@{\qquad} cc cc cc}
  \toprule
  \multirow{2}{*}{\raisebox{-\heavyrulewidth}{\textbf{Optimizer}}} &
  \multirow{2}{*}{\raisebox{-\heavyrulewidth}{\textbf{Depth}}} & \multicolumn{2}{c}{\textbf{Sch$\frac{1}{3}$}} & \multicolumn{2}{c}{\textbf{Sch$\frac{1}{2}$}} & \multicolumn{2}{c}{\textbf{Sch$\frac{2}{3}$}}  \\
  \cmidrule{3-8}
  &  & \textbf{Err} & \textbf{Rk} & \textbf{Err} & \textbf{Rk} & \textbf{Err} & \textbf{Rk} \\
  
  \toprule
  Adam \cite{Kingma:2015us} \\
  \hspace{3mm} w. amsgrad
                        & 1     & 0.99 & 69   & 0.97 & 11   & 0.55 & 7 \\
                        & 3     & 1.00 & 5    & 1.00 & 1   & 0.80 & 2\\
 \hspace{3mm} w.o. amsgrad
                        & 1     & 1.00 & 50   & 0.94 & 5   & 0.10 & 5\\
                        & 3     & 0.88 & 1   & 0.84 & 1   & 0.17 & 4\\ 
     
  \midrule
  \multirow{2}{*}{Adagrad \cite{duchi2011adaptive}}  & 1     & 1.00 & 12   & 0.97 & 2    & 0.70 & 15\\
                            & 3     & 0.81 & 1    & 0.85 & 1     & 0.08 & 5\\
  \midrule
  \multirow{2}{*}{Adamax \cite{Kingma:2015us}}  & 1     & 1.00 & 54   & 0.87 & 1    & 1.00 & 41 \\
                           & 3     & 1.00 & 23   & 0.85 & 1     & 1.00 & 4 \\
  \midrule
  \multirow{2}{*}{RMSProp} & 1     & 1.04 & 53   & 1.00 & 40     & 0.10 & 7 \\
                           & 3     & 1.00 & 8   & 0.74 & 1       & 0.08 & 5 \\  
  \midrule
  GD \\
 \hspace{3mm}      w. momentum         
                           & 1     & 266.7 & 43   & 2.14 & 29   & 0.65 & 16 \\
                           & 3     & 1.00 & 25    & 1.00 & 17    & 0.41 & 3\\  
\hspace{3mm}       w.o. momentum
                           & 1     & 67.6 & 45   & 0.94 & 3   & 0.77 & 19 \\
                           & 3     & 1.10 & 34    & 0.99 & 1    & 0.58 & 2\\  
\midrule
  \multirow{2}{*}{Adadelta \cite{adadelta}}  & 1     & 0.99 & 60   & 1.05 & 48   & 1.00 & 60\\
                             & 3     & 0.99 & 56    & 0.94 & 1    & 0.99 & 56\\                             
\midrule
  \multirow{2}{*}{AdamW \cite{adamw}}  & 1     & 1.00 & 48   & 0.97 & 9   & 0.67 & 6\\
                          & 3     & 1.00 & 10    & 0.99 & 1    & 0.82 & 2 \\        
\midrule                            
  \multirow{2}{*}{NAdam \cite{nadam}}  & 1     & 1.00 & 47   & 0.98 & 8   & 0.80 & 7\\
                          & 3     & 1.00 & 3    & 0.92 & 1    & 0.71 & 3\\                                    
\midrule
\multirow{2}{*}{RAdam \cite{radam}}  & 1     & 1.00 & 46    & 1.01 & 25    & 0.57 & 7\\
                        & 3     & 0.99 & 7     & 0.91 & 1     & 0.68 & 3 \\             
   
  \bottomrule
 \end{tabular}
 \label{tbl:diff_optsregs_more2}
 \end{center}
\end{table}

We note that Adam-like variants such as AdamW, NAdam, and RAdam also, unsurprisingly, exhibit varying degrees of depth invariance with respect to both generalization error and rank reduction with our ratio penalty though their solutions fail to generalize as well as Adam and Adamax under the penalty---this makes sense as Adam and Adamax are much more closely related, and both proposed in the same paper with only slight differences in their update rules, than subsequent variants that have significantly different iterative updates (see \cite{adamw, nadam, radam, Kingma:2015us} for further details).

\subsection{Supplemental Derivations}
\label{sec:math}
Here, we outline and detail the derivations behind some of our work describing the trajectory of the end-product matrix $W$ and its singular values $\sigma_i$ under gradient flow, largely building off of \cite{Arora:2018vn, Arora:2019ug}. We highlight the main derivations as they relate to the theorems in this paper and defer to the appendices of \cite{Arora:2018vn, Arora:2019ug} for additional details and a fuller treatment of gradient flow derivations. We also provide additional commentary and discussion here due to space constraints in the main body of the paper.  

\subsubsection{Theorem 2 (Gradient flow, with penalty)} 
\label{sec:gd_pen}
\paragraph{Setup \& preliminaries.} We omit the details of deriving the $\dot{\sigma}$ and $\dot{W}$ under un-regularized gradient flow; instead, we defer to the appendix of \cite{Arora:2018vn} for a detailed treatment and build upon their results for our derivations. 

\paragraph{Evolution of $W$ and $\sigma$ under gradient flow without penalty.} To begin, we recall the dynamics of the end-product matrix $W$ and its singular values $\sigma$ under \textit{un-regularized} gradient flow (for details on a full derivation, see the appendices of \cite{Arora:2019ug, Arora:2018vn} for a detailed treatment):
\begin{align}
\dot{\sigma}_i &= -N (\sigma_r^2)^{\frac{N-1}{N}} \vb{u}_r^\top \nabla_W \mathcal{L}(W) \vb{v}_r
\label{eqn:sv_again}
\\
\mathrm{vec}(\dot{W}) &= -P_{W} \mathrm{vec} \left(\nabla_W \mathcal{L}(W)\right)
\label{eqn:w_again}
\end{align}
Since the explicit regularization simply adds a regularization term to the original loss function, we can re-calculate the gradient of the new, regularized loss and substitute the gradients back into the dynamics above in order to re-characterize the above dynamics under the penalty. In more formal terms, if we let the new modified loss function be represented by $ \tilde{\mathcal{L}}(W) \coloneqq  \mathcal{L}(W) + \lambda R(W)$ where $R(W) = ||W||_* / ||W||_F$ denotes the explicit penalty, then we can calculate $ \nabla_W \tilde{\mathcal{L}}(W) = \nabla_W \mathcal{L}(W) + \lambda \nabla_W R(W)$ to plug back into the un-regularized dynamics in order to characterize the trajectories under the effect of the penalty.

\paragraph{Evolution of $W$ and $\sigma$ under gradient flow with penalty.} To characterize the dynamics of gradient flow in presence of the penalty, we first derive the gradient of the penalty (i.e. $\nabla_W R(W)$) ignoring the regularization strength (hyper)parameter $\lambda$. Via the chain rule and the sub-gradient of the nuclear norm, we can express the gradient of the penalty as:
\begin{align}
\nabla_W R(W) = \frac{1}{\vert \vert W \vert \vert_F ^2} \left(UV^\top - \frac{\vert \vert W \vert \vert_*}{\vert \vert W \vert \vert_F} U \Sigma V^\top \right)
\label{eqn:penalty_gradient}
\end{align}
where $U$, $V$, $\Sigma$ are the singular matrices of the singular value decomposition of $W$ (i.e. $W = U \Sigma V^\top$). 

To first characterize the dynamics of the end-product matrix $W$, we can use the gradient of the penalty and substitute it back into \cref{eqn:w_again} along with the loss gradient. Taking \cref{eqn:penalty_gradient}, plugging it back into \cref{eqn:w_again} along with the original loss gradient $\nabla_W \mathcal{L}(W)$, and re-arranging terms, we have 
\begin{align}
        \mathrm{vec}(\dot{W}) &= -P_W \mathrm{vec} \left(\nabla_W \mathcal{L}(W) + \lambda \nabla_W R(W) \right) \nonumber \\
        &= -P_W \left( \mathrm{vec} \left(\nabla_W \mathcal{L}(W)\right) + \lambda \frac{\mathrm{vec}(UV^\top - \frac{\vert \vert W \vert \vert_*}{\vert \vert W \vert \vert_F} U \Sigma V^\top)}{\vert \vert W \vert \vert_F^2} \right) \nonumber \\
        &= -P_W \left( \mathrm{vec} \left(\nabla_W \mathcal{L}(W)\right) + \lambda \frac{\mathrm{vec}(UV^\top - U \tilde{\Sigma} V^\top)}{\vert \vert W \vert \vert_F^2} \right) \nonumber 
\end{align}
where $\tilde{\Sigma} = \frac{\vert \vert W \vert \vert_*}{\vert \vert W \vert \vert_F} \Sigma$, producing the expression in \cref{eqn:gdreg_w}.

Turning now to the singular values, for analytical ease in characterizing the dynamics of the singular values ${\sigma_i}$ under gradient flow with the penalty (i.e., $\dot{\sigma}_r^{\textrm{reg}}$), we can equivalently express \cref{eqn:penalty_gradient} in terms of its diagonal elements; for instance, if we focus on the $r$-th diagonal element, this produces
\begin{align}
    \frac{\vb u_r^\top \vb v_r}{\vert \vert W \vert \vert_F ^2} \left(1 - \frac{\vert \vert W \vert \vert_*}{\vert \vert W \vert \vert_F}  \sigma_{r} \right) \nonumber 
\end{align}
where $\{\vb u_r, \vb v_r\}$ are the $i$-th left and right singular vectors of $W$ and $\sigma_{r}$ is the $r$-th diagonal element of $\Sigma$. Substituting this term along with the original loss gradient back into \cref{eqn:sv_again} and re-writing/grouping terms, we obtain our desired result in the form of \cref{eqn:gdreg_sv} for the $r$-th singular value, i.e.,
\[
  \dot{\sigma}_r^{\textrm{reg}} = \dot{\sigma}_r^{\textrm{GF}} 
   - \frac{\lambda N}{\vert \vert W \vert \vert_F^2} \left(1 - \frac{\vert\vert W \vert\vert_*}{\vert\vert W \vert\vert_F}\right) \sigma_r^{\frac{3N - 2}{2}} 
\]
where $\dot{\sigma}_r^{\textrm{GF}}$ is defined in \cref{eqn:sv_again} (i.e. the trajectory of singular values under un-regularized gradient flow).
\subsubsection{Theorem 1 (Adam, no penalty)} 
\label{sec:adam_nopen}

\paragraph{Setup \& preliminaries.}
We omit the details of deriving $P_W$ from the beginning and defer to the appendix of \cite{Arora:2018vn} for a more comprehensive review. Operations like divisions, squaring, and square-roots on vectors and matrices are to assumed to be performed element-wise unless otherwise stated. In this case, since parameters are matrices, loss gradients will generally be matrices; in this case, when referring to the variance of said objects, the variance is taken to mean the variance of the vectorized matrix (i.e. condensed into a vector) drawn from some distribution.

Per-layer weight updates under Adam in discrete time take the form, assuming no weight-decay:
\begin{equation*}
W_j^{(t+1)} = W_j^{(t)} - \alpha\frac{\hat{m}_t}{\sqrt{\hat{v}_t } + \varepsilon}
\end{equation*}
where $\hat{m}_t = \frac{m_t}{1-\beta_1^t}$, $m_t = (1 - \beta_1) m_{t-1} + \beta_1 g_{t-1}$, 
$\hat{v}_t = \frac{v_t}{1-\beta_2^t}$,
$v_t = (1 - \beta_2) v_{t-1} + \beta_2 g_{t-1}^2$ for each layer $j \in \{1,\hdots,N\}$ in a depth $N$ network. $\alpha$ is the initial learning rate, $\hat{m}_t$ is the bias-corrected first moment, $\hat{v}_t$ is the bias-corrected second moment, and $\varepsilon$ is a division-by-zero adjustment which we assume to be zero to simplify our later analysis. In turn, $m_t$ is the moving average of stochastic gradients $g_t$ and $v_t$ is the moving average of their element-wise square $g_t^2$, where $g_t = \nabla_{W_j} \mathcal{L}(W^{(t)})$, the gradient of the loss with respect to the $j$-th layer weight, $g_t^2$ is the element-wise squaring of $g_t$, and $\beta_1, \beta_2 \in (0,1)$ are discount factors for $m_t$ and $v_t$ respectively. Moving forward, for notational simplicity, we suppress the dependence of the $W$ and $W_j$ on $t$ unless otherwise stated and set $\beta_1 = \beta_2 = \beta = 1 - \epsilon$ for sufficiently small epsilon (e.g. $\epsilon = 0.001$). 

Given the loss in \cref{eqn:setup}, the gradient of the loss with respect to the $j$-th layer weight matrix is (suppressing notation on time $t$):
\begin{equation}
g \coloneqq \nabla_{W_j}\mathcal{L}(W) = \prod_{i = j + 1}^{N} W_i^\top  \left(\nabla_{W}\mathcal{L}(W)\right) \prod_{i = 1}^{j-1} W_i^\top
\label{eqn:layer_gradient}
\end{equation}
where $\nabla_{W}\mathcal{L}(W)$ is the gradient of the loss with respect to the end-product matrix $W$. While gradient flow uses this loss directly in its iterative updates for each $j$-th weight matrix, Adam makes adjustments to this loss gradient in its calculation of $\hat{m}$ and $\hat{v}$. 

\paragraph{Assumptions.} As noted earlier in \cref{sec:findings}, to approximately characterize the trajectory under Adam in closed form, albeit imperfectly, we make a few additional assumptions. More formally, in the same spirit as \cite{Adam:sgn}, we assume that the underlying distribution of stochastic gradients is stationary and approximately constant over the time horizon of the moving average of gradients so that $\hat{m}_t$ and $\hat{v}_t$ approximate estimate their respective moments of $g_t$. Namely:
\begin{equation}
\hat{m}_t \rightarrow \mathbb{E}(\hat{m}_t) \approx \nabla_{W_j} \mathcal{L}(W^{(t)}), \hspace{3mm} \hat{v}_t \rightarrow \mathbb{E}(\hat{v}_t) \approx (\nabla_{W_j} \mathcal{L}(W^{(t)}))^2 + s^2_{t,j}
\end{equation}
where $s^2_{t,j}$ is the (element-wise) variance of $\nabla_{W_j} \mathcal{L}(W^{(t)})$ with the assumption that each $W_j$ is drawn i.i.d. from a common stationary distribution. Naturally, this assumption can only hold approximately, at best, and is clearly inaccurate for gradients far in the past---to that end, we assume some sufficient degree of ``burn-in'' or progression so that past some time $t$, gradients near initialization or extremely distant past gradients do not contribute meaningfully to the moving average. As noted in \cite{Adam:sgn}, this assumed approximation is more realistic and more likely to hold in certain cases (e.g. high noise or small step size, the latter of which aligns with the spirit of gradient flow). 

\paragraph{Evolution of $W$ under Adam without penalty.} Under the assumptions above and those in \cite{Arora:2019ug}, to approximately characterize the trajectory of the end-product matrix, we can rewrite Adam's update of each weight layer $W_j$ using our assumptions above in continuous time (i.e. $\alpha \rightarrow 0$), in a fashion similar to gradient flow, as (suppressing notational dependence on time $t$):
\begin{align}
\dot{W}_j &= - \frac{\nabla_{W_j} \mathcal{L}(W)}{\sqrt{\nabla_{W_j} \mathcal{L}(W)^2 + s^2_j}} \label{eqn:adam_updateflow}
\end{align}
where, again, division, etc. here denote element-wise operations. As an aside, we note that the expression in \cref{eqn:adam_updateflow} can also be written as $\dot{W_j} = -(1 + \eta_j^2)^{-1/2}$ where $\eta^2_j \coloneqq {s^2_j}/{\nabla_{W_j} \mathcal{L}(W)^2}$, resembling a variance adaptation factor for each layer $j$ that shortens the update in directions of high (relative) variance as an adaptive means to account for varying reliability of the gradient throughout the optimization process or an approximate signal-to-noise ratio of the gradient \cite{Kingma:2015us, Adam:sgn}. Using the definition of $\nabla_{W_j} \mathcal{L}(W)$ from \cref{eqn:layer_gradient}, we can rewrite the right-hand side of \cref{eqn:adam_updateflow} as: 
\begin{equation}
- \prod_{i = j + 1}^{N} W_i^\top  (\nabla_{W} \mathcal{L}(W)) \prod_{i = 1}^{j-1} W_i^\top \odot S_j
\end{equation}
where $\odot$ denotes the Hadamard product and $S_j$ is a matrix that contains the corresponding entry-wise elements of $(\nabla_{W_j} \mathcal{L}(W)^2 + s^2_j)^{-1}$ to make the element-wise operations more explicit. To derive the evolution of the end-product matrix $\dot{W}$ from $\dot{W}_j$, we follow the approach in \cite{Arora:2019ug} but with several modifications due to presence of $S_j$. For the full details, we defer to the appendix of \cite{Arora:2019ug}. We left and right multiply by a product of other weight layers and sum across the depth of the network:
\begin{align}
\dot{W} = -\sum_{j=1}^N \prod_{j+1}^{i=N} W_i \left( \prod_{i = j + 1}^{N} W_i^\top  ((\nabla_{W} \mathcal{L}(W))) \prod_{i = 1}^{j-1} W_i^\top \odot S_j \right) \prod_1^{i = j-1} W_i
\end{align}
To simplify notation for the proceeding steps, we define the following:
\begin{align}
\nonumber A = \prod_{i = j + 1}^{N} W_i^\top, \
\nonumber B = \prod_{i = 1}^{j-1} W_i^\top, \
\nonumber L =  \nabla_{W} \mathcal{L}(W) \\
\nonumber \alpha = A^\top, \hspace{5mm}
\nonumber \beta = ALB \odot S_j, \hspace{5mm}
\nonumber \gamma = B^\top
\end{align}
Re-writing the above expression with these new definitions, we have:
\begin{align*}
 -\sum_{j=1}^N A^\top(ALB \odot S_j)B^\top = -\sum_{j=1}^N \alpha \beta \gamma 
\end{align*}
Ignoring the negative sign to focus on the summation and taking the (column-major order) vectorization of the above expression, we then have:
\begin{align}
 \mathrm{vec}\left(\sum_j \alpha \beta \gamma \right) &= \sum_j \mathrm{vec}(\alpha \beta \gamma) \nonumber \\ &=\sum_j (\gamma^\top \otimes \alpha) \mathrm{vec}(\beta) \nonumber  \\
 &=\sum_j (B \otimes A^\top) \mathrm{vec}(ALB \odot S_j) \nonumber  \\
 &=\sum_j (B \otimes A^\top)(\mathrm{vec}(ALB) \odot \mathrm{vec}(S_j)) \nonumber  \\
 &=\sum_j (B \otimes A^\top)(B^\top \otimes A)(\mathrm{vec}(L) \odot \mathrm{vec}(S_j)) \nonumber  \\
 &=\sum_j(BB^\top \otimes A^\top A)(\mathrm{vec}(L) \odot \mathrm{vec}(S_j)) \nonumber  \\
 &=\sum_{j=1}^N ((W W^\top)^{\frac{N-j}{N}} \otimes (W^\top W)^{\frac{j-1}{N}})(\mathrm{vec}(L) \odot \mathrm{vec}(S_j)) \nonumber 
\end{align}
where the first two equalities is due to the additivity of the vectorization operator and the association between vectorized products and the Kronecker product ($\mathrm{vec}(ABC) = (A^\top \otimes C) \mathrm{vec}(B)$), the fourth equality is due to the preservation of the Hadamard product over vectorization, the fifth equality is the result of applying the same logic as the first equality to $\mathrm{vec}(ALB)$, the sixth equality is the result of the mixed-product property of Kronecker products, and the simplification in the final equality can be found in the appendix of \cite{Arora:2018vn}. Previous work \cite{Arora:2018vn} has shown that, in un-regularized gradient flow, $P_W \coloneqq \sum_{j=1}^N ((W W^\top)^{\frac{N-j}{N}} \otimes (W^\top W)^{\frac{j-1}{N}})$ is a p.s.d. matrix that serves as an accelerative pre-conditioning that acts on the loss gradient. However, now there is an additional component $S_j$ so we can further simplify the previous expression:
\begin{align}
 &= \sum_{j=1}^N ((W W^\top)^{\frac{N-j}{N}} \otimes (W^\top W)^{\frac{j-1}{N}})(\mathrm{diag}(\mathrm{vec}(S_j)) \mathrm{vec}(L)) \nonumber \\
 &=\sum_{j=1}^N ((W W^\top)^{\frac{N-j}{N}} \otimes (W^\top W)^{\frac{j-1}{N}})(G_j \mathrm{vec}(L)) \nonumber 
\end{align}
where $G_j = \text{diag}(\text{vec}(S_j)$ denotes taking the elements of $\mathrm{vec}(S_j))$ and placing them along the diagonal of a zero matrix. Finally, substituting the term back in for $L$, we have a approximate characterization of the trajectory of the end-product matrix $W$ under Adam:
\begin{equation}
\mathrm{vec}(\dot{W}) = -P_{W, G} \mathrm{vec}(\nabla_{W} \mathcal{L}(W))
\label{eqn:adam_wtraj}
\end{equation}
where $P_{W, G} \coloneqq \sum_{j=1}^N ((W W^\top)^{\frac{N-j}{N}} \otimes (W^\top W)^{\frac{j-1}{N}})G_j$.

\paragraph{Positive semi-definiteness of $P_{W,G}.$}
Under gradient flow, \cite{Arora:2018vn} has shown that
$P_W = \sum_{j=1}^N P_j$ where $P_j = (W^\top W)^{\frac{N-j}{N}} \otimes (W W^\top)^{\frac{j-1}{N}}$ and $P_j$ is symmetric; to show that $P_j$ is p.s.d., note that, suppressing notation for time dependence $t$, we can equivalently write for some $j$,
\begin{equation}
P_j = A( (D D^\top)^{\frac{j-1}{N}} \otimes (D^\top D)^{\frac{N-j}{N}} )A^\top
\end{equation}
where $A = U \otimes V$, the Kronecker product of the matrices containing the left and right singular vectors of $W$ respectively, and $D$ is the diagonal matrix containing the singular values of $W$. Since the term in the middle is the Kronecker product of two diagonal matrices (with non-negative singular values along the diagonal), it's also a diagonal matrix. Therefore, for non-zero $x$ and some fixed $j$,
\[x^\top A((D D^\top)^{\frac{j-1}{N}} \otimes (D^\top D)^{\frac{N-j}{N}} )A^\top x \geq ||A^\top x||^2c \geq 0 \]
where $c = \sigma^{2\frac{j-1}{N}}_{1} \cdot \sigma^{2\frac{N-j}{N}}_{1}$ and $\sigma_{1}$ is the smallest singular value of $W$ or, equivalently, the smallest value of the diagonal matrix $D$ which is by definition non-negative. $G_j$ is also symmetric, since it is diagonal, and p.s.d. since its diagonal entries (and therefore eigenvalues) are by definition non-negative. Since $P_jG_j$ is the product of a real symmetric p.s.d. matrix $P_j$ and a real positive diagonal matrix $G_j$, the product is also p.s.d. \cite{Ramuzat2022}. Therefore, $P_{W,G} = \sum_{j=1}^N P_j G_j$ is also p.s.d. as a sum of p.s.d. matrices. Similar to $P_W$ under un-regularized gradient flow, $P_{W,G}$ acts as a pre-conditioning on the loss gradient which now also includes the variance of the loss gradient with respect to each layer $j$.

\paragraph{Implications of $P_{W,G}.$} As noted in \cite{Arora:2018vn}, $P_W$ in the original case of un-regularized gradient flow serves as an accelerative pre-conditioning on the loss gradient (\cref{eqn:gd_w}) that intensifies with depth and can only exist with sufficient depth ($N \geq 2$). More specifically, \cite{Arora:2018vn} show that the eigenvalues of $P_W$ are $\sum_{j=1}^N \sigma^{2\frac{N-j}{N}}_r \sigma^{2\frac{j-1}{N}}_{r'}$, corresponding to its eigenvectors $\mathrm{vec}(\vb u_r \vb v_{r'}^\top)$ where $\vb u_r$ and $\vb v_{r'}^\top$ are the left and right singular vectors of the end-product matrix $W \in \mathbb{R}^{m \times n}$. With non-trivial depth ($N \geq 2$), increasing $\sigma_r$ or $\sigma_{r'}$ results in an increase in the eigenvalue corresponding to the eigen-direction (i.e. rank one matrix) $\vb u_r \vb v_{r'}^\top$, which can be interpreted as the pre-conditioning favoring or accelerating in directions that correspond to singular vectors whose presence in the end-product matrix $W$ is stronger. In the case of degenerate depth ($N=1$), however, the pre-conditioning $P_W$ collapses into the identity, causing all of its non-zero eigenvalues to reduce down to unity---resulting in no accelerative favoring of any eigen-direction(s) during optimization. 

In the case of Adam's pre-conditioning at degenerate depth, $P_{W,G}$'s non-zero eigenvalues, unlike $P_W$, are no longer just unity and static which produce no accelerative favoring in any direction; instead, its eigenvalues become $[(1 + \eta^2)^{-1/2}]_{r,r'}$ where now $\eta^2 = s^2/\nabla_{W} \mathcal{L}(W)^2$, allowing for continued variation and accelerative effects even at depth 1. For non-degenerate depths ($N\geq2$), each individual $P_j$ within $P_{W,G}$ is now normalized by $G_j$ so that each layer's individual contribution to the overall pre-conditioning is normalized by a function of that layer's squared gradient and gradient variance, unlike before. 


\paragraph{Evolution of $\sigma$ under Adam without penalty.}  For the trajectories of the singular values $\{\sigma_i\}$, we largely build atop the approach in \cite{Arora:2019ug}, whose key steps we highlight here. We start by defining the analytic singular value decomposition of the end-product matrix as $W(t) = U(t)D(t)V(t)^\top$. Differentiating $W(t)$ with respect to time $t$:
\[ \dot{W}(t) = \dot{U}(t) D(t) V(t)^\top + U(t) \dot{D}(t) V(t)^\top + U(t) D(t)\dot{V}(t) \]
If we then left-multiply the above expression by $U(t)^\top$ and right-multiply by $V(t)$, we see that:
\[  U(t)^\top \dot{W}(t)V(t) =  U(t)^\top \dot{U}(t) D(t) +  \dot{D}(t) + D(t)\dot{V}(t)V(t) \]
where the orthonormal columns of $U(t)$ and $V(t)$ have helped simplify the earlier expression. Focusing on the diagonal elements of the above expression and suppressing the notation for time dependence, we see that they are:
\[
\vb u^\top_i \dot{W} \vb v_i = \vb u^\top_i \vb{\dot{u}}_i^\top \sigma_i + \dot{\sigma}_i + \sigma_i \vb{\dot{v}}_i^\top \vb{v}_i 
\]
where $\vb u_i$ and $\vb v_i$ are the $i$-th left and right singular vectors associated with the $i$-th diagonal element or, equivalently, the $i$-th singular value $\sigma_i$. Since the columns of $U$ are orthonormal by definition, and because $\vb u_i$ and $\vb v_i$ has constant unit length by definition (i.e. $\vb u_i^\top \dot{\vb u}_i = \frac{1}{2} \frac{d}{dt} \vert \vert \vb u_i(t) \vert \vert_2^2 = 0$), the above equation can be distilled down into:
\[
\dot{\sigma}_i = \vb u^\top_i \dot{W} \vb v_i
\]
Taking the (column-order first) vectorization of the expression above, then:
\begin{equation}
\mathrm{vec}(\dot{\sigma}_i) = \dot{\sigma}_i = (\vb u^\top_i \otimes \vb v^\top_i) \mathrm{vec}(\dot{W}) =  \mathrm{vec}(\vb v_i \vb u_i^\top)^\top \mathrm{vec}(\dot{W}) 
\label{eqn:vec_singularvalue}
\end{equation}
We can then substitute our expression for $\mathrm{vec}(\dot{W})$ from \cref{eqn:adam_wtraj} to characterize the singular values under Adam:
\begin{equation}
\dot{\sigma}_i =  - \mathrm{vec}(\vb v_i \vb u_i^\top)^\top  P_{W, G} \text{vec}(\nabla_{W} \mathcal{L}(W))
\end{equation}

\subsubsection{Lemma 4 (Adam, with penalty)} 
\label{sec:adam_pen}

We omit the details of deriving $\dot{\sigma}$ under gradient flow from the beginning and defer to the appendix of \cite{Arora:2018vn} for a more comprehensive review. We highlight the key parts.

\paragraph{Evolution of $W$ under Adam with penalty.} To characterize the evolution of the end-product matrix $W$ under Adam in presence of the penalty, we can simply leverage our earlier results from \cref{eqn:adam_wtraj} and combine them with the gradient of the penalty since the penalty is simply an additive component to the same loss function. Denoting our penalty by $R(W)$, we can characterize $\dot{W}$ as:
\[
\mathrm{vec}(\dot{W}) = -P_{W, G} \text{vec}(\nabla_{W} \mathcal{L}(W) + \lambda \nabla_{W} R(W))
\]
Substituting in the gradient of the penalty and vectorizing accordingly, we have:
\begin{equation}
\mathrm{vec}(\dot{W}) = -P_{W,G}\left(\mathrm{vec} \left(\nabla_W \mathcal{L}(W)\right) + \lambda \frac{\mathrm{vec}(UV^\top - U \tilde{\Sigma} V^\top)}{\vert \vert W \vert \vert_F^2}\right)
\label{eqn:adam_wtraj_penalty}
\end{equation}
where $U$ and $V^\top$ are the matrices containing the left and right singular vectors of the end-product matrix $W$ at time $t$ (time notation suppressed in the expression above), $\lambda$ is the regularization strength, and $\tilde{\Sigma}$ is a re-weighted version of a rectangular diagonal matrix containing the singular values of $W$ (i.e. $\tilde{\Sigma} = \frac{\vert \vert W \vert \vert_*}{\vert \vert W \vert \vert_F} \Sigma$). Here, we can see that $P_{W,G}$ in \cref{eqn:adam_wtraj_penalty} is accelerating not just the loss gradient in helping with optimization with respect to performance via the loss but also the penalty gradient in its tendency towards low-rankedness. 

Interestingly, we also note that the second term in the parenthesis of \cref{eqn:adam_wtraj_penalty} can be seen as a re-scaled version of $W$. Namely, if we ignore the vectorization operator, the term $(UV^{\top} -  U \tilde{\Sigma} V^\top)/\vert \vert W \vert \vert_F^2$ can be re-expressed as a new spectrally shifted and re-scaled $W$ that we can define as $\Bar{W} \coloneqq U\Bar{\Sigma}V^\top$ where $\Bar{\Sigma} = (I - \tilde{\Sigma})/\vert \vert W \vert \vert_F^2$. In other words, we see that the explicit regularizer affects the network's (i.e., $W$) trajectory by introducing a new spectrally adjusted version of itself as part of the penalization.

\paragraph{Evolution of $\sigma$ under Adam with penalty.} Similar to the approach taken to characterize $\dot{W}$ under the penalty, we can re-use the expression in \cref{eqn:vec_singularvalue} and substitute in the appropriate expression for $\mathrm{vec}(\dot{W})$ in the case of Adam with the penalty:  
\begin{equation}
\dot{\sigma}_i =  - \mathrm{vec}(\vb u_i \vb v_i^\top)^\top  P_{W, G} \left( \text{vec}(\nabla_{W} \mathcal{L}(W)) + \lambda \frac{\mathrm{vec}(UV^\top - U \tilde{\Sigma} V^\top)}{\vert \vert W \vert \vert_F^2} \right)
\end{equation}

\subsubsection{Depth 1 dynamics}
Lastly, we clarify the dynamics of the end-product matrix $W$ and its singular values $\sigma$ in the case of a depth 1 network. 

\paragraph{Un-regularized gradient flow and Adam.} As mentioned in \cref{sec:findings}, for un-regularized gradient flow, at $N=1$, the trajectories of the end-product matrix and its singular values reduce down to:
\begin{align}
\dot{\sigma}_i &= - \vb{u}_i^\top \nabla_W \mathcal{L}(W) \vb{v}_i \\
\text{vec}(\dot{W}) &= - \text{vec} \left(\nabla_W \mathcal{L}(W)\right)
\end{align}
where $P_{W} = I_{mn}$ has now reduced down to the identity; as such, the accelerative pre-conditioning that typically strengthens with depth no longer exists, as expected.

For un-regularized Adam, at $N=1$, we have:
\begin{align}
\dot{\sigma}_i &= -\mathrm{vec}(\vb v_i \vb u_i^\top)^\top (G \cdot \mathrm{vec} (\nabla_W \mathcal{L}(W))) \\
\mathrm{vec}(\dot{W}) &= -G \cdot \mathrm{vec} (\nabla_W \mathcal{L}(W))
\end{align}
where $P_{W,G} = I_{mn}$, $G_j = G$ since now $N = j = 1$, $G =  \mathrm{diag}(\mathrm{vec}(S_j))$, and $S_j$ is defined as a matrix whose elements are: $[S_j]_{m,n} = [(\nabla_{W} \mathcal{L}(W)^2 + s^2)^{-1/2}]_{m,n}$ as defined in \cref{eqn:adam_w}. While the depth-dependent accelerative pre-conditioning still no longer exists, we have a new sort of ``pre-conditoning'' in the form of a p.s.d. matrix $G$ that is a function of the squared loss gradient and its variance.

\paragraph{Regularized gradient flow and Adam.} Finally, we describe the dynamics for the end-product matrix and its singular values under gradient flow and Adam at depth 1, but now with the penalty. 

For gradient flow with the penalty, we see that:
\begin{align}
\dot{\sigma}_r &= -\vb{u}_i^\top \nabla_W \mathcal{L}(W) \vb{v}_i 
- \frac{\lambda}{\vert \vert W \vert \vert_F^2} \left(1 - \frac{\vert\vert W \vert\vert_*}{\vert\vert W \vert\vert_F}\right) \sigma_r^{\frac{1}{2}} \\
\mathrm{vec}(\dot{W}) &= -\left(\mathrm{vec} \left(\nabla_W \mathcal{L}(W)\right) + \lambda \frac{\mathrm{vec}(UV^\top - U \tilde{\Sigma} V^\top)}{\vert \vert W \vert \vert_F^2}\right)
\end{align}
As noted earlier, in the absence of any pre-conditioning, we do have an additional degree of freedom provided by the penalty both in terms of the evolution of $\dot{W}$ (i.e. via the penalty gradient and not just the loss gradient) and $\dot{\sigma}$ (i.e. to depend on its own relative magnitude unlike un-regularized gradient flow).

For Adam, we have:
\begin{align}
\dot{\sigma}_i &=  -\mathrm{vec}(\vb v_i \vb u_i^\top)^\top  G \left(\mathrm{vec} \left(\nabla_W \mathcal{L}(W)\right) + \lambda \frac{\mathrm{vec}(UV^\top - U \tilde{\Sigma} V^\top)}{\vert \vert W \vert \vert_F^2}\right) \\
\mathrm{vec}(\dot{W}) &= -G \left(\mathrm{vec} \left(\nabla_W \mathcal{L}(W)\right) + \lambda \frac{\mathrm{vec}(UV^\top - U \tilde{\Sigma} V^\top)}{\vert \vert W \vert \vert_F^2}\right) 
\end{align}
As noted earlier in \cref{sec:findings}, the key is the combination of $G$, a pre-conditioning that is a function of the variance of the loss gradient and appears under Adam and Adam-like variants, and the gradient penalty from our normalized nuclear norm ratio, that allows a depth 1 network (i.e. no depth) to generalize as well as a deep network, or deep factorization, and perform the same extent of rank reduction, as gradient flow/descent's implicit regularization in the presence of depth, to the point of perfect rank recovery as measured by effective rank like in \cite{Arora:2019ug}.



\end{document}